\pdfoutput=1 \documentclass[twoside,11pt]{article}
\newif\ifarxiv
\arxivtrue

\usepackage{amsmath,amssymb,amsthm,mathtools,bm}
\let\amsproof\proof
\let\endamsproof\endproof
\usepackage[preprint,nohyperref]{jmlr2e}
\let\proof\amsproof
\let\endproof\endamsproof
\usepackage[utf8]{inputenc}
\usepackage[T1]{fontenc}
\usepackage{graphicx,xcolor,booktabs}
\usepackage{tikz}
\usetikzlibrary{fit}
\usepackage{pgfplots}
\pgfplotsset{compat=1.18}
\usepgfplotslibrary{groupplots}
\usepackage{enumitem}
\usepackage{algorithm,algorithmic}
\usepackage{hyperref}
\usepackage{url}
\hypersetup{colorlinks=true,linkcolor=blue,citecolor=blue,urlcolor=blue,
  pdftitle={Embedding-Bias in Conditional Independence Testing},
  pdfauthor={Nikolaj Thams and Anton Rask Lundborg}}

\usepackage[capitalize,nameinlink,noabbrev]{cleveref}
\crefname{assumption}{Assumption}{Assumptions}
\Crefname{assumption}{Assumption}{Assumptions}
\crefname{algorithm}{Algorithm}{Algorithms}
\Crefname{algorithm}{Algorithm}{Algorithms}
\crefname{appendix}{Appendix}{Appendices}
\Crefname{appendix}{Appendix}{Appendices}
\crefname{subappendix}{Appendix}{Appendices}
\Crefname{subappendix}{Appendix}{Appendices}
\crefname{equation}{Equation}{Equations}
\Crefname{equation}{Equation}{Equations}

\makeatletter
\@for\@tempa:=theorem,lemma,proposition,remark,corollary,definition,conjecture,axiom,example\do{\expandafter\let\csname\@tempa\endcsname\relax
  \expandafter\let\csname end\@tempa\endcsname\relax}
\makeatother
\newtheorem{theorem}[theorem]{Theorem}
\numberwithin{theorem}{section}
\newtheorem{proposition}[theorem]{Proposition}
\newtheorem{lemma}[theorem]{Lemma}
\newtheorem{corollary}[theorem]{Corollary}
\newtheorem{definition}[theorem]{Definition}
\newtheorem{assumption}[theorem]{Assumption}
\newtheorem{example}[theorem]{Example}
\newtheorem*{examplecontinner}{Example~\ref{\examplecontlabel}}
\newenvironment{examplecont}[1]{\def\examplecontlabel{#1}\begin{examplecontinner}[continued]}{\end{examplecontinner}}
\theoremstyle{remark}
\newtheorem{remark}[theorem]{Remark}

\newcommand{\indep}{\mathrel{\perp\!\!\!\perp}}

\newcommand{\R}{\mathbb{R}}
\newcommand{\E}{\mathbb{E}}
\DeclareMathOperator{\Var}{Var}
\DeclareMathOperator{\Cov}{Cov}
\DeclareMathOperator{\Corr}{Corr}
\DeclareMathOperator{\id}{id}
\DeclareMathOperator{\RV}{RV}
\newcommand{\rbench}{\rho_{\psi \to \psi'}}
\newcommand{\rbenchhat}{\hat\rho_{\psi \to \psi'}}

\newcommand{\resOneDimZ}{20}

\newcommand{\resOneQinsuff}{5}

\newcommand{\resOneNrep}{500}

\newcommand{\resOneRtwoA}{0.1}

\newcommand{\resOneBound}{0.1}

\newcommand{\resOnePowerN}{1000}
\newcommand{\resOnePowerBudgets}{\{0.05, 0.1, 0.2\}}
\newcommand{\resOnePowerBudgetRef}{0.1}
\newcommand{\resOnePowerNaive}{0.90}
\newcommand{\resOnePowerRobust}{0.04}

\newcommand{\resOneRVTrue}{0.1}

\newcommand{\resOneRVCoverage}{95\%}

\newcommand{\resTwoN}{15{,}521}
\newcommand{\resTwoPopN}{8000}
\newcommand{\resTwoNrep}{500 (200 for the high-dimensional encoder)}
\newcommand{\resTwoNmin}{500}
\newcommand{\resTwoNmax}{10{,}000}
\newcommand{\resTwoDimWeak}{384}
\newcommand{\resTwoDimStrong}{768}

\newcommand{\resTwoRichDim}{1152}

\newcommand{\resTwoTrainRowsSmall}{400}

\newcommand{\resTwoOLSFromRich}{2{,}000}

\newcommand{\resTwoPopRtwo}{0.109}
\newcommand{\resTwoPopBound}{0.109}
\newcommand{\resTwoPopBench}{0.109}
\newcommand{\resTwoPopRtwoKRR}{0.095}
\newcommand{\resTwoPopBenchKRR}{0.092}

\newcommand{\resTwoTypeIOrthSmall}{0.51}
\newcommand{\resTwoTypeIOrthOneK}{0.22}
\newcommand{\resTwoTypeIOrthFromTwoKMin}{0.07}
\newcommand{\resTwoTypeIOrthFromTwoKMax}{0.10}

\newcommand{\resTwoNrepBench}{100}
\newcommand{\resTwoNrepPower}{100}

\newcommand{\resTwoTypeIRichRidgeLarge}{0.73}
\newcommand{\resTwoRhoRichRidgeLarge}{0.026}

\newcommand{\resTwoRVmed}{0.103}
\newcommand{\resTwoRVmedSmall}{0.118}
\newcommand{\resTwoRValphaSmall}{0.039}
\newcommand{\resTwoRValphaLarge}{0.087}
\newcommand{\resTwoBenchSampleMed}{0.000}

\newcommand{\resTwoPowerN}{2000}

\newcommand{\resTwoTypeIWeakPowerN}{0.96}
\newcommand{\resTwoRobustBoundNull}{0.08}

\newcommand{\resThreeModel}{Claude Haiku 4.5}

\newcommand{\resThreeN}{3000}

\newcommand{\resThreeTruncFrac}{2.9\%}

\newcommand{\resAppNops}{3{,}051}

\newcommand{\resFourModel}{Qwen3-1.7B}
\newcommand{\resFourParams}{1.7-billion}
\newcommand{\resFourLayers}{28}
\newcommand{\resFourHidden}{2048}
\newcommand{\resFourMidLayer}{14}

\newcommand{\resFourGenHours}{3.4}
\newcommand{\resFourN}{2{,}999}
\newcommand{\resFourMargNull}{0.516}

\newcommand{\resFourRhoEone}{0.215}

\newcommand{\resFourRhoEtwo}{0.232}

\newcommand{\resDeepN}{49{,}050}
\newcommand{\resDeepNparents}{18{,}743}
\newcommand{\resDeepNposts}{2{,}653}
\newcommand{\resDeepMaxPerParent}{8}
\newcommand{\resDeepRowsPerParent}{2.6}
\newcommand{\resDeepDeltaRate}{1.4\%}

\newcommand{\resDeepParentWords}{78}
\newcommand{\resDeepParentDepthOne}{16\%}
\newcommand{\resDeepRatingCost}{1.0}
\newcommand{\resDeepRated}{49{,}050}
\newcommand{\resDeepPolStrat}{0.094}

\newcommand{\resDeepPolEone}{0.072}

\newcommand{\resDeepPolEmbMin}{0.062}
\newcommand{\resDeepPolEmbMax}{0.072}

\newcommand{\resDeepPolPredRho}{0.072}

\newcommand{\resDeepLenEone}{0.091}

\newcommand{\resDeepLenPredRho}{0.093}

\newcommand{\resDeepBenchMax}{0.024}

\newcommand{\resDeepParentFit}{84\%}
\newcommand{\resDeepAgreeN}{4{,}728}
\newcommand{\resDeepAgreeCorr}{0.80}
\newcommand{\resDeepAgreeSpearman}{0.79}
\newcommand{\resDeepAgreeExact}{45\%}
\newcommand{\resDeepAgreeWithinOne}{97\%}

\newcommand{\resDeepRaterMedian}{2.57}
\newcommand{\resDeepRaterSd}{0.76}
\newcommand{\resDeepRaterAtHaikuOne}{1.21}
\newcommand{\resDeepRaterAtHaikuFive}{4.21}

\newcommand{\resDeepLambdaCompMin}{-0.32}
\newcommand{\resDeepLambdaCompMax}{0.23}

\newcommand{\resFourPowerP}{20}
\newcommand{\resFourPowerR}{200}
\newcommand{\resFourPowerRange}{0.84}

\newcommand{\resFourPowerFactorEone}{0.47}

\newcommand{\resFourPowerRhoDiffMaxEone}{0.01}

\newcommand{\resFourTruthM}{20}
\newcommand{\resFourTruthN}{2{,}997}
\newcommand{\resFourTruthR}{200}

\newcommand{\resFourTruthPredRhoEone}{0.218}

\newcommand{\resFourTruthLambdaMinStates}{0.64}
\newcommand{\resFourTruthLambdaMaxStates}{0.72}

\newcommand{\resFourTruthPredDiffMax}{0.006}

\newcommand{\figorboxh}[3]{\IfFileExists{#1}{\includegraphics[width=#2]{#1}}{\fbox{\parbox{\dimexpr#2-2\fboxsep-2\fboxrule\relax}{\centering\rule{0pt}{#3}\\{\scriptsize\ttfamily\detokenize{#1}}\\{\scriptsize (figure pending)}}}}}
\newcommand{\figorbox}[2]{\figorboxh{#1}{#2}{3.1cm}}
\newcommand{\figorboxtall}[2]{\figorboxh{#1}{#2}{4.4cm}}

\ShortHeadings{Embedding-Bias in Conditional Independence Testing}{Thams and Lundborg}
\firstpageno{1}

\begin{document}

\title{Embedding-Bias in Conditional Independence Testing}

\author{\name Nikolaj Thams \email thams@math.ku.dk \\
       \addr Department of Mathematical Sciences\\
       University of Copenhagen\\
       Copenhagen, Denmark
       \AND
       \name Anton Rask Lundborg \email arl@math.ku.dk \\
       \addr Department of Mathematical Sciences\\
       University of Copenhagen\\
       Copenhagen, Denmark}

\maketitle
\thispagestyle{plain}

\begin{abstract}  To test conditional independence of $X$ and $Y$ given a text or an image $Z$,
  one conditions on an embedding $\psi(Z)$ in place of $Z$. The embedded test is
  valid if $Z$ is independent of $X$ or of $Y$ given $\psi(Z)$, which cannot be
  confirmed from data, and when this fails, the rejection probability under the
  null hypothesis can tend to one. We study this failure, and show that focusing on a specific form of
  dependence relaxes what the embedding must retain.  For a residual correlation
  test inspired by the Generalised Covariance Measure, validity only requires
  that the parts of $\mathbb{E}[X \mid Z]$ and $\mathbb{E}[Y \mid Z]$ missed by
  $\mathbb{E}[X \mid \psi(Z)]$ and $\mathbb{E}[Y \mid \psi(Z)]$ are
  uncorrelated. Otherwise, we treat the discarded information as an omitted
  variable. Under the null hypothesis, the bias equals the absolute correlation
  of the missed parts times the geometric mean of two partial $R^2$ values. This
  identity yields a robust test valid under a declared tolerance for the
  geometric mean, which, like a sensitivity parameter, is not identified from the data. 
  On synthetic data and text embeddings, the robust test holds
  its level approximately. On text generated by a language model, under an exact null
  hypothesis, every embedding, even the generator's own states, biases the embedded
  test. 
 \end{abstract}


\section{Introduction}\label{sec:intro}

Tests of conditional independence (CI) are a basic tool of
statistics~\citep{shah2020}. They underpin tests of variable
significance~\citep{lundborg2024pcm}, the evaluation of predictors for
fairness~\citep{hardt2016equality} and robustness~\citep{peters2016causal}, and
constraint-based causal discovery~\citep{spirtes2000causation}. In many current
applications, the tested variables are text, images, or other unstructured
objects, and the test is run on embeddings of those objects instead
\citep{veitch2020,kook2024, simnacher2026dncit,simnacher2026ecit}, testing $H_0: X \indep Y \mid Z$ by testing the surrogate hypothesis $H_0^\psi:
\psi_X(X) \indep \psi_Y(Y) \mid \psi_Z(Z)$ for embeddings $\psi_X, \psi_Y$ and
$\psi_Z$. 

Even a perfect test of $H_0^{\psi}$ may not result in a valid test of $H_0$. The problem lies not with embedding $X$ or $Y$; testing $H_0$ via a test
of $\psi_X(X) \indep \psi_Y(Y) \mid Z$ is valid, though possibly less powerful
than a direct test of $H_0$. Embedding the conditioning variable $Z$, however, does not generally preserve conditional independence. In this paper, we
formalize and quantify this potential loss of validity. Only the embedding
of $Z$ matters for validity, so we absorb $\psi_X$ and $\psi_Y$ into $X$ and $Y$
and write $H_0^\psi$ for $X \indep Y \mid \psi(Z)$, with $\psi :=
\psi_Z$.

\begin{example}
Consider an observational study of online discussions. Let $Z$ be an original post, 
$X$ be a reply, and $Y$ the votes the reply receives. Whether the style of a reply is unassociated with its reception other than through the content of the post is the hypothesis $X \indep Y \mid Z$. To test it, one reduces the reply to a summary such as its politeness $f(X)$ and the votes to a score $g(Y)$, and replaces the post by an embedding $\psi(Z)$ from a pretrained encoder (\cref{fig:idea}). Suppose that the encoder represents the topic of a post well but its opinion poorly, and that the opinion influences both the politeness of the replies and their votes. Then the parts of $f(X)$ and $g(Y)$ that are not explained by $\psi(Z)$ are correlated even if $X \indep Y \mid Z$ holds, and every consistent test on the embedded variables rejects with probability tending to one. 
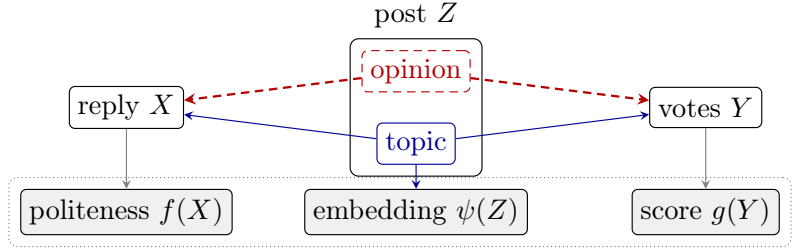
\begin{figure}[t]
\centering
\resizebox{\ifarxiv 0.68\textwidth\else\columnwidth\fi}{!}{\begin{tikzpicture}[>=stealth, every node/.style={font=\small},
  v/.style={draw, rounded corners=2pt, minimum height=5mm, inner sep=3pt, fill=white},
  e/.style={v, fill=gray!12},
  keep/.style={v, draw=blue!60!black, text=blue!60!black},
  miss/.style={v, draw=red!70!black, text=red!70!black, densely dashed}]
  \node[miss] (O) at (0, 0.5) {opinion};
  \node[keep] (T) at (0, -0.4) {topic};
  \node[draw, rounded corners=3pt, inner sep=4pt, fit=(O)(T), label={[font=\small]above:post $Z$}] (Z) {};
  \node[v] (X) at (-3.6, 0.05) {reply $X$};
  \node[v] (Y) at (3.6, 0.05) {votes $Y$};
  \node[e] (pZ) at (0, -1.25) {embedding $\psi(Z)$};
  \node[e] (pX) at (-3.6, -1.25) {politeness $f(X)$};
  \node[e] (pY) at (3.6, -1.25) {score $g(Y)$};
  \draw[->, blue!60!black] (T) -- (X);
  \draw[->, blue!60!black] (T) -- (Y);
  \draw[->, red!70!black, thick, densely dashed] (O) -- (X);
  \draw[->, red!70!black, thick, densely dashed] (O) -- (Y);
  \draw[->, blue!60!black] (T) -- (pZ);
  \draw[->, gray] (X) -- (pX);
  \draw[->, gray] (Y) -- (pY);
  \node[draw=gray, densely dotted, rounded corners=3pt, inner sep=4pt, fit=(pX)(pZ)(pY)] (box) {};
\end{tikzpicture}} \caption{The conceptual example: the topic and the opinion of a post $Z$ both influence its reply $X$ and its reception $Y$, and the embedded test uses $f(X)$, $g(Y)$, and $\psi(Z)$. Since $\psi(Z)$ represents the topic but not the opinion, the dashed path through the opinion stays open and $X$ and $Y$ are dependent given $\psi(Z)$ although $X \indep Y \mid Z$.} 
\label{fig:idea}
\end{figure}
\end{example}

Existing embedded tests apply a standard CI test to the embedded variables and
rely on the embedding being \emph{sufficient}. \citet{simnacher2026ecit} show
that this is valid if $X \indep Z \mid \psi(Z)$ or $Y \indep Z \mid \psi(Z)$
(their Theorem~3) or, for tests of conditional mean
independence of $Y$, if $\E[Y \mid Z] = \E[Y \mid \psi(Z)]$ (their Theorem~14). A pretrained encoder, trained
for a generic objective, has no reason to retain what $Z$ says about $X$ or $Y$,
and \citet[Section~3.2]{simnacher2026ecit} themselves consider even the mean
condition implausible for learned embeddings and assess it by simulation. Nor
can sufficiency be checked from the data. It is itself a conditional (mean)
independence hypothesis given $\psi(Z)$, for which no test is uniformly valid
with nontrivial power \citep{shah2020,neykov2021,lundborg2024pcm}; a richer
model of $Z$ can only show that $\psi(Z)$ misses part of $\E[X \mid Z]$.

We examine how embedded tests fail without strict sufficiency, and how far the failure can be contained, in two steps. First, we choose a target that requires less of the embedding. 
In particular, we consider a family of conditional association measures that includes the targets of commonly used CI tests. 
Each measure is the largest absolute correlation between residuals of
transformations of $X$ and of $Y$ after adjusting for the conditioning variable,
and the classes of transformations determine which departures from $H_0$ it
captures. For the absolute correlation $\rho_\psi$ of the residuals of $X$ and
$Y$ after adjusting for $\psi(Z)$, it suffices that the parts of $\E[X \mid Z]$
and $\E[Y \mid Z]$ that $\psi(Z)$ misses are uncorrelated, for instance because
$\psi(Z)$ retains one of the two. However the test has no power
against dependence that leaves the residuals given $Z$ uncorrelated. 

Second, as in sensitivity analysis for unmeasured confounding
\citep{rosenbaum2002,imbens2003,cinelli2020,chernozhukov2022lss}, we treat the
information that $\psi$ discards as an omitted variable and quantify its effect.
Under $H_0$, $\rho_\psi = |\lambda| \sqrt{R^2_\psi(X)\, R^2_\psi(Y)}$, where
$R^2_\psi(X)$ is the share of the residual variance of $X$ given $\psi(Z)$ that
$Z$ would explain and $\lambda$ is the correlation of the two missed parts. This
yields a test that is valid, whatever the correlation of the discarded parts,
under a declared bound on $\sqrt{R^2_\psi(X)\, R^2_\psi(Y)}$. Because the
discarded information is part of the observed $Z$, one can use data and any
richer embedding to lower-bound $R^2_\psi(X)$ and $R^2_\psi(Y)$.

Our contributions are as follows.
In \cref{sec:setup}, we define a family of measures of conditional dependence and bound how embedding $Z$ changes them. 
In \cref{sec:sens}, we treat the discarded information as an omitted variable and derive a test that remains valid under a declared bound on the resulting bias.
In \cref{sec:exp}, we evaluate the bias on synthetic and real data. On text generated by a language model under an exact null hypothesis, every embedding we tried biases the plain embedded test.

\subsection{Related work}\label{sec:related}

\textbf{CI testing and embeddings.} Besides covariance measures \citep{shah2020,scheidegger2022,lundborg2024pcm,kook2024}, CI tests use kernel mean embeddings \citep{zhang2011kci,bergen2026gkcm}, the conditional law of $X$ given $Z$ \citep{candes2018,berrett2020,tansey2022}, local permutations \citep{kim2022local}, or generative models \citep{zhang2024doubly,ren2025score}. 

\citet{simnacher2026ecit}, discussed above, develop CI tests for embedded text. 
A number of works \citep{keith2020,feder2022,egami2022} treat sufficiency as an untestable assumption, to be argued from domain knowledge and examined by sensitivity analysis or simulation. 
In our work, we quantify how a lack of sufficiency invalidates the test.

\textbf{Sensitivity analysis.} Sensitivity analysis bounds the bias that an unobserved confounder can cause \citep{rosenbaum2002,imbens2003}. \citet{cinelli2020} parametrize it by partial $R^2$ values in linear regression, \citet{chernozhukov2022lss} extend the bounds to linear functionals of a nonparametric regression, and \citet{veitch2020sense} apply them post hoc; \cref{app:ovb} compares our identity with these bounds. Conditioning on $\psi(Z)$ is conditioning on a coarsening of the confounder, for which \citet{ogburn2012,ogburn2013} give bias-attenuation results.

\subsection{Notation}\label{sec:notation}
Let $(X, Y, Z) \sim P$, where 
$X$, $Y$, and $Z$ take values in measurable spaces $\mathcal X$, $\mathcal Y$,
and $\mathcal Z$, for example spaces of texts. We observe $n$ independent copies
$(X_i, Y_i, Z_i)$. We write $\E$, $\Var$, and $\Cov$ for moments under $P$,
$\E[U \mid T]$ for the conditional expectation of $U$ given $T$, and $U \indep T \mid S$ for conditional independence of $U$ and $T$ given $S$. Equalities between random variables hold almost surely.

The hypothesis of interest is $H_0\colon X \indep Y \mid Z$. An embedding $\psi\colon \mathcal Z \to \mathcal V$, with $\mathcal V$ a
measurable space such as $\R^d$, is a fixed measurable map, pretrained on
external data or fitted on a split of the sample that is not used for testing,
and $\id$ denotes the identity on $\mathcal Z$. Every quantity below is indexed
by the transformation of $Z$ that is conditioned on.
For a real-valued, square-integrable function $f$ of $(X, Z)$ and a measurable $\psi$, including $\psi = \id$, let
\[
\sigma_\psi^2(f) := \E\big[(f - \E[f \mid \psi(Z)])^2\big] = \E[\Var(f \mid \psi(Z))]
\]
be the residual variance of $f = f(X, Z)$ after adjusting for $\psi(Z)$. When
$X$ is real-valued and $f(x, z) = x$, we write $X$ in place of $f$, as in
$\sigma_\psi^2(X)$. Throughout, when a definition or result is symmetric in $(X,
f)$ and $(Y, g)$, we state it for $X$ and $f$ only. All proofs are in
\cref{app:proofs}.
 \section{Measures of conditional dependence}\label{sec:setup}

\subsection{Definition and examples}\label{sec:embedded}

In general, neither $H_0$ nor $H_0^\psi$ implies the other. Rather than comparing the two hypotheses, we
compare measures of conditional dependence given $Z$ and given $\psi(Z)$, which
vanish under $H_0$ and under $H_0^\psi$, respectively. Write $P_{X,Z}$ and
$P_{Y,Z}$ for the distributions of $(X, Z)$ and of $(Y, Z)$, and let $\mathcal F
\subseteq L^2(P_{X,Z})$ and $\mathcal G \subseteq L^2(P_{Y,Z})$ be classes of
real-valued transformations of $(X, Z)$ and of $(Y, Z)$; we write $f$ for $f(X,
Z)$. A test that conditions on $\psi(Z)$ can only use transformations that
depend on $Z$ through $\psi(Z)$. We therefore let $\mathcal F_\psi$ be the $f
\in \mathcal F$ for which there exists an $h$ such that $f(x, z) = h(x, \psi(z))$,
and $\mathcal G_\psi$ likewise. We remark that $\mathcal F_{\id} = \mathcal F$.
If $\mathcal{F}$ is a class of functions of $x$ alone, such as $\{(x, z) \mapsto x\}$, then $\mathcal F_\psi = \mathcal F$.

\begin{definition}[Measure of conditional dependence]\label[definition]{def:family}
For $f \in \mathcal F_\psi$ and $g \in \mathcal G_\psi$, let
\[
\rho_\psi(f, g) := \frac{|\E[\Cov(f, g \mid \psi(Z))]|}{\sigma_\psi(f)\, \sigma_\psi(g)},
\]
with $\rho_\psi(f, g) := 0$ if $\sigma_\psi(f) = 0$ or $\sigma_\psi(g) = 0$, and let $\rho_\psi(\mathcal F, \mathcal G) := \sup_{f \in \mathcal F_\psi,\, g \in \mathcal G_\psi} \rho_\psi(f, g)$, with $\rho_\psi(\mathcal F, \mathcal G) := 0$ if $\mathcal F_\psi$ or $\mathcal G_\psi$ is empty.
\end{definition}

The measure $\rho_\psi(f, g)$ is the absolute correlation of the residuals $f -
\E[f \mid \psi(Z)]$ and $g - \E[g \mid \psi(Z)]$. Under $H_0^{\psi}$, $f \indep
g \mid \psi(Z)$ for $f \in \mathcal F_\psi$ and $g \in \mathcal G_\psi$, since
$f$ is a function of $(X, \psi(Z))$ and $g$ one of $(Y, \psi(Z))$, so
$\rho_{\psi}(f, g) = 0$ for all $f$ and $g$, and $\rho_{\psi}(\mathcal F,
\mathcal G) = 0$ for every choice of classes. The classes determine which forms
of conditional dependence the measure captures, and larger classes capture more.

\begin{example}\label[example]{example:family}
\begin{enumerate}[label=(\alph*),leftmargin=*,itemsep=1pt,topsep=2pt]
\item If $\mathcal F = L^2(P_{X,Z})$ and $\mathcal G = L^2(P_{Y,Z})$, then $\rho_\psi(\mathcal F, \mathcal G) = 0$ if and only if $X \indep Y \mid \psi(Z)$.
\item If $Y$ is real-valued with $\E[Y^2] < \infty$, $\mathcal F = L^2(P_{X,Z})$,
and $\mathcal G = \{(y, z) \mapsto y\}$, then $\rho_\psi(\mathcal F, \mathcal G)
= 0$ if and only if $\E[Y \mid X, \psi(Z)] = \E[Y \mid \psi(Z)]$.
\item If $X$ and $Y$ are real-valued with $\max\{\E[X^2], \E[Y^2]\} < \infty$,
$\mathcal F = \{(x, z) \mapsto x\}$, and $\mathcal G = \{(y, z) \mapsto y\}$,
then $\rho_\psi(\mathcal F, \mathcal G) = |\E[\Cov(X, Y \mid \psi(Z))]|/(\sigma_\psi(X)\, \sigma_\psi(Y))$, the absolute correlation of
the residuals $X - \E[X \mid \psi(Z)]$ and $Y - \E[Y \mid \psi(Z)]$.
\end{enumerate}
\end{example}

In particular, if the measure in (a) vanishes, so does that in (b), and if that in (b) vanishes, so does that in (c). 
Reducing a text-valued or vector-valued tested variable to a real-valued summary $f(X)$, such as a coordinate of an embedding, is the choice $\mathcal F = \{(x, z) \mapsto f(x)\}$. This choice will be our main focus in applications, but we develop our theory for general $\mathcal{F}, \mathcal{G}$.

\subsection{Coarsening the conditioning variable}\label{sec:retain}

For $f \in \mathcal F_\psi$, let $\Delta_\psi(f) := \E[f \mid Z] - \E[f \mid \psi(Z)]$ be the part of the regression of $f$ on $Z$ that the embedding does not represent. Since $f - \E[f \mid \psi(Z)] = (f - \E[f \mid Z]) + \Delta_\psi(f)$ with uncorrelated summands, $\sigma_\psi^2(f) = \sigma_{\id}^2(f) + \E[\Delta_\psi(f)^2]$. Define
\[
R^2_\psi(f) := \frac{\E[\Delta_\psi(f)^2]}{\sigma_\psi^2(f)} = 1 - \frac{\sigma_{\id}^2(f)}{\sigma_\psi^2(f)} \in [0, 1],
\]
the partial $R^2$ of $f$ with respect to $\psi(Z)$. 
If $\sigma_\psi(f) = 0$, then $f = \E[f \mid
\psi(Z)]$, so $\Delta_\psi(f) = 0$, and we set $R^2_\psi(f) := 0$. Let
$\lambda_\psi(f, g) := \Corr(\Delta_\psi(f), \Delta_\psi(g))$, with
$\lambda_\psi(f, g) := 0$ under degeneracy. The following
result relates $\rho_\psi$ to $\rho_{\id}$.

\begin{theorem}[Bias bound]\label{thm:pair}
For all $f \in \mathcal F_\psi$ and $g \in \mathcal G_\psi$,
\ifarxiv
\[
\Big|\rho_\psi(f, g) - \sqrt{(1 - R^2_\psi(f))(1 - R^2_\psi(g))}\;\rho_{\id}(f, g)\Big| \le |\lambda_\psi(f, g)| \sqrt{R^2_\psi(f)\, R^2_\psi(g)}.
\]
\else
\[
\begin{split}
\Big|\rho_\psi(f, g) - &\sqrt{(1 - R^2_\psi(f))(1 - R^2_\psi(g))}\;\rho_{\id}(f, g)\Big| \\ &\le |\lambda_\psi(f, g)| \sqrt{R^2_\psi(f)\, R^2_\psi(g)}.
\end{split}
\]
\fi
Under $H_0$, $\rho_\psi(f, g) = |\lambda_\psi(f, g)| \sqrt{R^2_\psi(f)\,
R^2_\psi(g)}$.
\end{theorem}

The factor $\sqrt{(1 - R^2_\psi(f))(1 - R^2_\psi(g))}$ shrinks $\rho_{\id}(f,
g)$ since the residuals given $\psi(Z)$ have larger variance than those given $Z$. The right-hand side of the inequality in \cref{thm:pair} is the absolute correlation of the discarded parts
times the geometric mean of the two partial $R^2$ values. It vanishes if the
embedding discards nothing about $f$ or about $g$, or if the discarded parts are uncorrelated. Taking suprema over
the classes yields the following corollary.

\begin{corollary}[Bias bound for classes]\label[corollary]{cor:sup}
Let $R^2_\psi(\mathcal F) := \sup_{f \in \mathcal F_\psi} R^2_\psi(f)$, with $R^2_\psi(\mathcal F) := 0$ if $\mathcal F_\psi$ is empty.
\begin{enumerate}[label=(\roman*),leftmargin=*,itemsep=1pt,topsep=2pt]
\item $\rho_\psi(\mathcal F, \mathcal G) \le \rho_{\id}(\mathcal F, \mathcal G) + \sqrt{R^2_\psi(\mathcal F)\, R^2_\psi(\mathcal G)}$.
\item Under $H_0$,
$$ 
\rho_\psi(\mathcal F, \mathcal G) = \sup_{f \in \mathcal F_\psi,\, g \in \mathcal G_\psi} |\lambda_\psi(f, g)| \sqrt{R^2_\psi(f)\, R^2_\psi(g)},$$ so $\rho_\psi(\mathcal F, \mathcal G) = 0$
if and only if $\lambda_\psi(f, g) = 0$ for all $f \in \mathcal F_\psi$ and $g \in \mathcal G_\psi$.
\end{enumerate}
\end{corollary}

By (i), whether or not $H_0$ holds, an embedded measure exceeding $\sqrt{R^2_\psi(\mathcal F) R^2_\psi(\mathcal G)}$ implies dependence given $Z$ of at least the excess. By (ii), under $H_0$ the embedding creates no dependence exactly when the discarded parts are uncorrelated for all pairs, which a vanishing partial $R^2$ ensures but does not require.

\begin{examplecont}{example:family}
Under $H_0$, \cref{cor:sup}(i) gives $\rho_\psi(\mathcal F, \mathcal G) \le \sqrt{R^2_\psi(\mathcal F) R^2_\psi(\mathcal G)}$. Since $R^2_\psi(\mathcal F)$ is a supremum, this bound tightens from (a) to (c), which orders the three measures by what they require of the embedding. By \cref{lem:zeror2}, $R^2_\psi(L^2(P_{X,Z})) = 0$ if and only if $X \indep Z \mid \psi(Z)$, and $R^2_\psi(\{(x, z) \mapsto x\}) = 0$ if and only if $\psi$ is \emph{mean-sufficient} for $X$, that is, $\E[X \mid Z] = \E[X \mid \psi(Z)]$. Under $H_0$, the measure in (a) therefore vanishes if $X \indep Z \mid \psi(Z)$ or $Y \indep Z \mid \psi(Z)$, the condition of \citet[Theorem~3]{simnacher2026ecit}; that in (b) if $X \indep Z \mid \psi(Z)$ or $\psi$ is mean-sufficient for $Y$, the latter being the condition of \citet[Theorem~14]{simnacher2026ecit}; and that in (c) if $\psi$ is mean-sufficient for $X$ or for $Y$. By \cref{cor:sup}(ii), these conditions are sufficient but not necessary.
\end{examplecont}

The gap between these conditions can be strict, as the following example shows for the measures in \cref{example:family}(a) and (c).

\begin{example}[Discarded parts and their correlation]\label[example]{ex:hetero}
Let $Z = (Z_1, \dots, Z_4)$ have independent standard normal coordinates,
$\psi(Z) = Z_1$, and
\[
\begin{split}
X &= Z_1 + a Z_2 + Z_4\, \varepsilon_X, \\ Y &= Z_1 + b\big(\lambda Z_2 + \sqrt{1 - \lambda^2}\, Z_3\big) + Z_4\, \varepsilon_Y,
\end{split}
\]
with $a, b \ge 0$, $\lambda \in [-1, 1]$, and independent standard normal noise,
so that $H_0$ holds.
Given $Z_1$, $X$ and $Y$ share the scale $Z_4$, so $\Cov((X - Z_1)^2, (Y -
Z_1)^2 \mid Z_1) = 2(1 + a^2 b^2 \lambda^2) > 0$; hence $H_0^\psi$ fails and the
measure in \cref{example:family}(a) is positive, for all $a$, $b$, and
$\lambda$. For $a = b = 0$, $\psi$ is mean-sufficient for $X$ and $Y$ although
not sufficient, and the measure in \cref{example:family}(c) is zero. For $a, b >
0$, $\psi$ is not even mean-sufficient: $R^2_\psi(X) = a^2/(1 + a^2)$,
$R^2_\psi(Y) = b^2/(1 + b^2)$, and the discarded parts $\Delta_\psi(X) = a Z_2$
and $\Delta_\psi(Y) = b(\lambda Z_2 + \sqrt{1 - \lambda^2}\, Z_3)$ have
correlation $\lambda$. By \cref{thm:pair}, the measure in
\cref{example:family}(c) equals $|\lambda| \sqrt{R^2_\psi(X)\, R^2_\psi(Y)}$. It
is zero at $\lambda = 0$, where the embedding misses two unrelated aspects of
$Z$, one relevant for $X$ and one for $Y$, and largest at $|\lambda| = 1$, where
it misses the same aspect for both.
\end{example}

From now on, we consider $X$ and $Y$ that are real-valued and square integrable, possibly after reduction to summaries, and we focus on the measure in \cref{example:family}(c), $\rho_\psi := \rho_\psi(X, Y)$. Its classes are the smallest of the three, so it requires the least of the embedding: by \cref{cor:sup}, $\rho_\psi$ vanishes under $H_0$ exactly when the discarded parts $\Delta_\psi(X)$ and $\Delta_\psi(Y)$ are uncorrelated, which includes the case where $\psi$ is mean-sufficient for $X$ or for $Y$, and is otherwise bounded by $\sqrt{R^2_\psi(X) R^2_\psi(Y)}$, the smallest of the bounds in \cref{cor:sup}(i). The disadvantage of using this
measure is that it captures the fewest forms of conditional dependence:
$\rho_{\id}$ can vanish although $H_0$ fails.

 \section{Testing with an insufficient embedding}\label{sec:sens}

\subsection{Treating discarded information as an omitted variable}\label{sec:bias}

A sensitivity analysis expresses the bias caused by an omitted variable through
a few interpretable sensitivity parameters and asks how large they would have to
be to change a conclusion \citep{rosenbaum2002,imbens2003,cinelli2020}. The information in $Z$ that $\psi$ discards acts as an omitted
confounder of $X$ and $Y$ given $\psi(Z)$. It enters $X$ and $Y$ through the
discarded parts $\Delta_\psi(X)$ and $\Delta_\psi(Y)$.
Let $\lambda := \lambda_\psi(X, Y)$ be the correlation of the discarded parts. 
Conditioning on $\psi(Z)$
instead of $Z$ adds the covariance of the discarded parts,
\ifarxiv
\begin{equation}\label{eq:pair}
\E[\Cov(X, Y \mid \psi(Z))] = \E[\Cov(X, Y \mid Z)] + \lambda \sqrt{R^2_\psi(X)\, R^2_\psi(Y)}\; \sigma_\psi(X)\, \sigma_\psi(Y),
\end{equation}
\else
\begin{multline}\label{eq:pair}
\E[\Cov(X, Y \mid \psi(Z))] = \E[\Cov(X, Y \mid Z)] \\ + \lambda \sqrt{R^2_\psi(X)\, R^2_\psi(Y)}\; \sigma_\psi(X)\, \sigma_\psi(Y),
\end{multline}
\fi
which is the signed identity behind \mbox{\cref{thm:pair}}. The partial $R^2$
values and $\lambda$ are thus sensitivity parameters, and the bound
$\sqrt{R^2_\psi(X) R^2_\psi(Y)}$ on the bias under $H_0$ has the product form
that is familiar from omitted variable bias in linear regression
(\mbox{\cref{app:ovb}}).

Whether the discarded parts are uncorrelated (\cref{cor:sup}(ii)) cannot be
checked without estimating $\E[X \mid Z]$ and $\E[Y \mid Z]$, which may be
infeasible when only embedded variables are available.

\subsection{Testing with uncorrelated discarded parts}\label{sec:gcm}

We test $H_0$ through $\rho_\psi$. Our starting point is to apply the GCM
test of \citet{shah2020} with $\psi(Z)$ as the conditioning variable. To
perform the test, we regress $X$ and $Y$ on $\psi(Z)$ and reject $H_0$ when the
products of the residuals have a mean far from zero.
By \cref{thm:pair}, $\rho_\psi = 0$ under $H_0$ whenever $\lambda = 0$, so a
valid test of $\rho_\psi = 0$ is then valid for $H_0$.
The test we use, which we call the residual correlation test, differs from the
GCM test in two ways. 

First, we use cross-fitting for the regressions; the sample is split into $K$ folds, and for each observation $i$, $\hat m_X(\psi(Z_i))$ and $\hat
m_Y(\psi(Z_i))$ are predictions of $\E[X \mid \psi(Z)]$ and $\E[Y \mid \psi(Z)]$
from regressions fitted without the fold containing $i$. With the residuals
$\hat r_{X,i} := X_i - \hat m_X(\psi(Z_i))$ and $\hat r_{Y,i} := Y_i - \hat
m_Y(\psi(Z_i))$ and $R_i := \hat r_{X,i}\hat r_{Y,i}$, the GCM
statistic is
\[
T := \frac{n^{-1/2} \sum_{i=1}^n R_i}{\hat s_R},
\]
where $\hat s_R^2$ is the sample variance of the $R_i$. 
Cross-fitting increases computational cost, but in exchange, the asymptotic behaviour of the statistic can be controlled even when $H_0^\psi$ fails, which happens under $H_0$ whenever the discarded parts are correlated.

Second, instead of $T$ we use the sample version of the signed residual correlation $\rho^\pm_\psi := \E[\Cov(X, Y \mid
\psi(Z))]/(\sigma_\psi(X)\, \sigma_\psi(Y))$, whose absolute value is
$\rho_\psi$,
\begin{equation}\label{eq:rhohat}
\hat\rho^\pm_\psi := \frac{\hat s_R}{\sqrt n\, \hat\sigma_\psi(X)\, \hat\sigma_\psi(Y)}\, T = \frac{\frac1n \sum_{i=1}^n R_i}{\hat\sigma_\psi(X)\, \hat\sigma_\psi(Y)}.
\end{equation}
Here $\hat\sigma_\psi^2(X) := \frac1n\sum_i \hat r_{X,i}^2$ and
$\hat\sigma_\psi^2(Y) := \frac1n\sum_i \hat r_{Y,i}^2$ are the sample residual
variances. We write $\hat\rho_\psi := |\hat\rho^\pm_\psi|$ for the sample
version of $\rho_\psi$. 

By studying the correlation, we can do the sensitivity analysis in terms of the partial $R^2$ values, and the tolerance in \cref{sec:robust} below can be interpreted in terms of prediction accuracy. However, this adds the requirement that $\sigma_\psi(X)$ and $\sigma_\psi(Y)$ be estimated. 
When $\rho^\pm_\psi\neq 0$, their estimation errors enter at first order, which is why
\cref{ass:product} asks for each mean squared error to be $o_P(n^{-1/2})$ rather
than only their product to be $o_P(n^{-1})$, as for the GCM \citep{shah2020}.

\begin{assumption}\label[assumption]{ass:product} Both cross-fitted regressions
have mean squared error $o_P(n^{-1/2})$, that is, $\frac1n\sum_{i=1}^n \big(\hat
m_X(\psi(Z_i)) - \E[X_i \mid \psi(Z_i)]\big)^2 = o_P(n^{-1/2})$, and likewise
for $Y$. The residuals $X - \E[X \mid \psi(Z)]$ and $Y - \E[Y \mid \psi(Z)]$
have finite fourth moments and positive variances.
\end{assumption}
The assumption concerns estimation error, which vanishes as $n$ grows. 
Define $\xi_X := (X -
\E[X \mid \psi(Z)])/\sigma_\psi(X)$ and $\xi_Y$ similarly, and let
\[
v_\psi^2 := \Var\Big(\xi_X \xi_Y - \frac{\rho^\pm_\psi}{2}\,(\xi_X^2 + \xi_Y^2)\Big),
\]
with sample version $\hat v_\psi^2$ (see \eqref{eq:vhat} in \cref{app:normal}).

\begin{proposition}[Asymptotic normality]\label[proposition]{prop:normal}
Suppose that \cref{ass:product} holds and that $v_\psi^2 > 0$. Then
\[
\frac{\sqrt n\, (\hat\rho^\pm_\psi - \rho^\pm_\psi)}{\hat v_\psi} \to N(0, 1)
\]
in distribution, and $\hat v_\psi^2 \to v_\psi^2$ in probability.
\end{proposition}

The residual correlation test has the $p$-value $p_0 := 2\big(1 - \Phi(\sqrt n\,
\hat\rho_\psi/\hat v_\psi)\big)$ and rejects $H_0$ at level $\alpha$ when $p_0 <
\alpha$. When $\rho_\psi = 0$, it is asymptotically equivalent to the GCM test.
Under $H_0$, it has asymptotic level $\alpha$ if the discarded parts are
uncorrelated and rejects with probability tending to one if they are correlated
(\cref{ex:hetero}).

\subsection{Testing under an insufficiency tolerance}\label{sec:robust}

If an analyst is willing to assume, before testing, that $\sqrt{R^2_\psi(X) R^2_\psi(Y)} \le B$ for some $B \ge 0$, then $\rho_\psi \le B$ under $H_0$ by \cref{cor:sup}(ii). We call $B$ an insufficiency tolerance, and call it valid if this bound holds.
Like a sensitivity parameter, its validity cannot be verified
from the data, but it can be phrased in terms of prediction accuracy.

\begin{example}[An insufficiency tolerance]\label[example]{ex:tolerance}
Suppose the analyst is willing to assume that the best prediction of $X$ from the embedding, $\E[X \mid \psi(Z)]$, has a mean squared error (MSE) at most $10\%$ larger than that of the best prediction from the text itself, $\E[X \mid Z]$, that is, $\sigma_\psi^2(X) \le 1.1\, \sigma_{\id}^2(X)$, and likewise for $Y$. Since $R^2_\psi(X) = 1 - \sigma_{\id}^2(X)/\sigma_\psi^2(X)$, this bounds $R^2_\psi(X)$ and $R^2_\psi(Y)$ by $B = 1 - 1/1.1 \approx 0.091$. More generally, MSEs inflated by the factors $1 + \delta_X$ and $1 + \delta_Y$ give the tolerance $B = \sqrt{\delta_X \delta_Y / ((1 + \delta_X)(1 + \delta_Y))}$.
\end{example}

By \cref{prop:normal}, $\hat\rho^\pm_\psi$ is approximately $N(\rho^\pm_\psi, v_\psi^2/n)$. For every $r \ge 0$, the probability that $\hat\rho_\psi$ exceeds $r$ increases with $\rho_\psi$ (\cref{app:robust}), so we can construct a robust test by computing the $p$-value at the least favourable point $\rho_\psi = B$,
\begin{equation}\label{eq:calib}
p_B := 2 - \Phi\Big(\frac{\sqrt n\, (\hat\rho_\psi - B)}{\hat v_\psi}\Big) - \Phi\Big(\frac{\sqrt n\, (\hat\rho_\psi + B)}{\hat v_\psi}\Big),
\end{equation}
the probability of exceeding $\hat\rho_\psi$ in absolute value at that point, with $v_\psi$ estimated by $\hat v_\psi$. The robust test rejects $H_0$ when $p_B < \alpha$; at $B = 0$, it is the residual correlation test.

\begin{proposition}[Asymptotic validity under the tolerance]\label[proposition]{prop:robust}
Under the conditions of \cref{prop:normal} and $H_0$, if $B$ is a valid tolerance, then $\limsup_{n\to\infty} P(p_B < \alpha) \le \alpha$ for every $\alpha \in (0, 1)$, with equality iff $\rho_\psi = B$.
\end{proposition}

If the bias due to the embedding is treated as an unknown bias bounded by $B$, the robust test is
the bias-aware test of \citet{armstrong2018,armstrong2021}, with the least
favourable calibration used for partially identified parameters
\citep{imbens2004,stoye2009}. It detects every alternative with $\sqrt{(1 -
R^2_\psi(X))(1 - R^2_\psi(Y))}\,\rho_{\id} > B + \sqrt{R^2_\psi(X)
R^2_\psi(Y)}$, whatever $\lambda$ (\cref{cor:power} in \cref{app:power}).

Inverting the robust test gives a robustness value, a term we take from
\citet{cinelli2020}: $\RV_\alpha$, the smallest tolerance at which the robust
test no longer rejects, is the tolerance at which $p_B = \alpha$ (\cref{app:rv}).

\begin{example}[\Cref{ex:tolerance}, continued]\label[example]{ex:tolerance2}
Suppose the analyst observes $\hat\rho_\psi = 0.12$ with standard error $\hat v_\psi/\sqrt n = 0.02$. The residual correlation test has $p_0 = 2(1 - \Phi(6)) < 10^{-8}$ and rejects. With $B \approx 0.091$, the robust test has $p_B = 2 - \Phi(1.45) - \Phi(10.55) \approx 0.07$ and does not reject at level $0.05$; it stops rejecting once the tolerance exceeds $\RV_\alpha \approx 0.087$. 
\end{example}

\subsection{Benchmarking with richer embeddings}\label{sec:bench} Because $Z$ is
observed, the partial $R^2$ values have lower bounds that can be estimated from
the data, without the assumption on the relative strength of observed and
omitted variables that benchmarks for unmeasured confounding require
\citep{cinelli2020}. For a second embedding $\psi'$, let $R^2_{\psi \to
\psi'}(X)$ be the share of the residual variance of $X$ given $\psi(Z)$ that
$\psi'(Z)$ explains, and call
$\rbench := \sqrt{R^2_{\psi \to \psi'}(X)\, R^2_{\psi \to \psi'}(Y)}$ the
insufficiency of $\psi$ demonstrated by $\psi'$.

\begin{proposition}[Benchmark lower bound]\label[proposition]{prop:bench}
Suppose that $\sigma_\psi(X), \sigma_\psi(Y) > 0$. For every measurable $\psi'$,
$\rbench \le \sqrt{R^2_\psi(X) R^2_\psi(Y)}$, with equality if $\E[X \mid Z]$
and $\E[Y \mid Z]$ are functions of $(\psi(Z), \psi'(Z))$.
\end{proposition}

Every valid tolerance is therefore at least $\rbench$. Since the robust test
with tolerance $B$ rejects if and only if $B < \RV_\alpha$ (\cref{app:rv}), a
rejection with $\RV_\alpha \le \rbench$ could be explained by insufficiency that is
demonstrably present, and a rejection with $\RV_\alpha > \rbench$ is robust to
the demonstrated insufficiency. By \cref{prop:bench}, accurate predictions of
$X$ and of $Y$ from $Z$ demonstrate all of the insufficiency. We can therefore
take $\psi'(Z)$ to be such predictions, known as prognostic scores
\citep{hansen2008}, fitted on a split of the sample that is not used for testing
\citep{veitch2020}. Because they are one-dimensional, their demonstrated
insufficiency is estimated reliably at moderate $n$, unlike that of a richer
embedding with hundreds of coordinates (\cref{app:impl-bench}). In practice, $\rbench$
is replaced by its estimate $\rbenchhat$, or, for a strict version of the rule,
by a lower confidence bound.
 \section{Experiments}\label{sec:exp} We evaluate the sensitivity analysis in four settings of increasing realism: fully synthetic, semi-synthetic, LLM-generated, and real data.
We fit conditional means by OLS when $n \leq \text{dim}(\psi(Z)$, and by ridge regression otherwise. 

\begin{figure}[t]
    \centering
    \resizebox{\ifarxiv 0.85\textwidth\else\columnwidth\fi}{!}{\definecolor{cInk}{HTML}{52514E}
\definecolor{cMuted}{HTML}{898781}
\definecolor{cAxis}{HTML}{C3C2B7}
\definecolor{cGrid}{HTML}{E1E0D9}
\begin{tikzpicture}[baseline=(current bounding box.south)]
\node[anchor=south, font=\scriptsize, inner sep=1pt] at (1.550,0.030) {Synthetic (Experiment~1)};
\node[anchor=east, font=\scriptsize, text=cInk, inner sep=1pt] at (-0.120,-0.180) {standard, sufficient $\psi$};
\fill[color={rgb,255:red,191;green,216;blue,248}] (0.025,-0.335) rectangle (0.595,-0.025);
\node[font=\scriptsize, text=cInk, inner sep=0pt] at (0.310,-0.180) {0.05};
\fill[color={rgb,255:red,195;green,220;blue,249}] (0.645,-0.335) rectangle (1.215,-0.025);
\node[font=\scriptsize, text=cInk, inner sep=0pt] at (0.930,-0.180) {0.04};
\fill[color={rgb,255:red,191;green,216;blue,248}] (1.265,-0.335) rectangle (1.835,-0.025);
\node[font=\scriptsize, text=cInk, inner sep=0pt] at (1.550,-0.180) {0.05};
\fill[color={rgb,255:red,192;green,217;blue,248}] (1.885,-0.335) rectangle (2.455,-0.025);
\node[font=\scriptsize, text=cInk, inner sep=0pt] at (2.170,-0.180) {0.05};
\fill[color={rgb,255:red,195;green,220;blue,249}] (2.505,-0.335) rectangle (3.075,-0.025);
\node[font=\scriptsize, text=cInk, inner sep=0pt] at (2.790,-0.180) {0.04};
\node[anchor=east, font=\scriptsize, text=cInk, inner sep=1pt] at (-0.120,-0.540) {standard, sufficient $\psi$, 1152-dim.};
\node[anchor=east, font=\scriptsize, text=cInk, inner sep=1pt] at (-0.120,-0.900) {standard, insufficient $\psi$, $\lambda = 0$};
\fill[color={rgb,255:red,192;green,217;blue,248}] (0.025,-1.055) rectangle (0.595,-0.745);
\node[font=\scriptsize, text=cInk, inner sep=0pt] at (0.310,-0.900) {0.05};
\fill[color={rgb,255:red,196;green,220;blue,249}] (0.645,-1.055) rectangle (1.215,-0.745);
\node[font=\scriptsize, text=cInk, inner sep=0pt] at (0.930,-0.900) {0.03};
\fill[color={rgb,255:red,191;green,216;blue,248}] (1.265,-1.055) rectangle (1.835,-0.745);
\node[font=\scriptsize, text=cInk, inner sep=0pt] at (1.550,-0.900) {0.05};
\fill[color={rgb,255:red,195;green,219;blue,249}] (1.885,-1.055) rectangle (2.455,-0.745);
\node[font=\scriptsize, text=cInk, inner sep=0pt] at (2.170,-0.900) {0.04};
\fill[color={rgb,255:red,193;green,218;blue,248}] (2.505,-1.055) rectangle (3.075,-0.745);
\node[font=\scriptsize, text=cInk, inner sep=0pt] at (2.790,-0.900) {0.04};
\node[anchor=east, font=\scriptsize, text=cInk, inner sep=1pt] at (-0.120,-1.260) {standard, insufficient $\psi$, $\lambda = 1$};
\fill[color={rgb,255:red,43;green,121;blue,215}] (0.025,-1.415) rectangle (0.595,-1.105);
\node[font=\scriptsize, text=white, inner sep=0pt] at (0.310,-1.260) {0.58};
\fill[color={rgb,255:red,19;green,70;blue,135}] (0.645,-1.415) rectangle (1.215,-1.105);
\node[font=\scriptsize, text=white, inner sep=0pt] at (0.930,-1.260) {0.89};
\fill[color={rgb,255:red,13;green,55;blue,108}] (1.265,-1.415) rectangle (1.835,-1.105);
\node[font=\scriptsize, text=white, inner sep=0pt] at (1.550,-1.260) {1.00};
\fill[color={rgb,255:red,13;green,54;blue,107}] (1.885,-1.415) rectangle (2.455,-1.105);
\node[font=\scriptsize, text=white, inner sep=0pt] at (2.170,-1.260) {1.00};
\fill[color={rgb,255:red,13;green,54;blue,107}] (2.505,-1.415) rectangle (3.075,-1.105);
\node[font=\scriptsize, text=white, inner sep=0pt] at (2.790,-1.260) {1.00};
\node[anchor=east, font=\scriptsize, text=cInk, inner sep=1pt] at (-0.120,-1.620) {robust, $B$ = true insufficiency};
\fill[color={rgb,255:red,189;green,215;blue,247}] (0.025,-1.775) rectangle (0.595,-1.465);
\node[font=\scriptsize, text=cInk, inner sep=0pt] at (0.310,-1.620) {0.06};
\fill[color={rgb,255:red,193;green,218;blue,248}] (0.645,-1.775) rectangle (1.215,-1.465);
\node[font=\scriptsize, text=cInk, inner sep=0pt] at (0.930,-1.620) {0.04};
\fill[color={rgb,255:red,192;green,217;blue,248}] (1.265,-1.775) rectangle (1.835,-1.465);
\node[font=\scriptsize, text=cInk, inner sep=0pt] at (1.550,-1.620) {0.05};
\fill[color={rgb,255:red,194;green,218;blue,248}] (1.885,-1.775) rectangle (2.455,-1.465);
\node[font=\scriptsize, text=cInk, inner sep=0pt] at (2.170,-1.620) {0.04};
\fill[color={rgb,255:red,192;green,217;blue,248}] (2.505,-1.775) rectangle (3.075,-1.465);
\node[font=\scriptsize, text=cInk, inner sep=0pt] at (2.790,-1.620) {0.05};
\node[anchor=east, font=\scriptsize, text=cInk, inner sep=1pt] at (-0.120,-1.980) {robust, $B$ = half of it};
\fill[color={rgb,255:red,115;green,170;blue,237}] (0.025,-2.135) rectangle (0.595,-1.825);
\node[font=\scriptsize, text=cInk, inner sep=0pt] at (0.310,-1.980) {0.31};
\fill[color={rgb,255:red,69;green,142;blue,230}] (0.645,-2.135) rectangle (1.215,-1.825);
\node[font=\scriptsize, text=cInk, inner sep=0pt] at (0.930,-1.980) {0.46};
\fill[color={rgb,255:red,30;green,95;blue,175}] (1.265,-2.135) rectangle (1.835,-1.825);
\node[font=\scriptsize, text=white, inner sep=0pt] at (1.550,-1.980) {0.73};
\fill[color={rgb,255:red,15;green,61;blue,120}] (1.885,-2.135) rectangle (2.455,-1.825);
\node[font=\scriptsize, text=white, inner sep=0pt] at (2.170,-1.980) {0.95};
\fill[color={rgb,255:red,13;green,54;blue,107}] (2.505,-2.135) rectangle (3.075,-1.825);
\node[font=\scriptsize, text=white, inner sep=0pt] at (2.790,-1.980) {1.00};
\node[anchor=south, font=\scriptsize, inner sep=1pt] at (5.100,0.030) {Text (Experiment~2)};
\fill[color={rgb,255:red,191;green,217;blue,248}] (3.575,-0.335) rectangle (4.145,-0.025);
\node[font=\scriptsize, text=cInk, inner sep=0pt] at (3.860,-0.180) {0.05};
\fill[color={rgb,255:red,187;green,214;blue,247}] (4.195,-0.335) rectangle (4.765,-0.025);
\node[font=\scriptsize, text=cInk, inner sep=0pt] at (4.480,-0.180) {0.07};
\fill[color={rgb,255:red,194;green,218;blue,248}] (4.815,-0.335) rectangle (5.385,-0.025);
\node[font=\scriptsize, text=cInk, inner sep=0pt] at (5.100,-0.180) {0.04};
\fill[color={rgb,255:red,195;green,219;blue,249}] (5.435,-0.335) rectangle (6.005,-0.025);
\node[font=\scriptsize, text=cInk, inner sep=0pt] at (5.720,-0.180) {0.04};
\fill[color={rgb,255:red,187;green,214;blue,247}] (6.055,-0.335) rectangle (6.625,-0.025);
\node[font=\scriptsize, text=cInk, inner sep=0pt] at (6.340,-0.180) {0.07};
\fill[color={rgb,255:red,57;green,135;blue,229}] (3.575,-0.695) rectangle (4.145,-0.385);
\draw[cInk, dash pattern=on 1.1pt off 0.9pt, line width=0.45pt] (3.587,-0.683) rectangle (4.133,-0.397);
\node[font=\scriptsize, text=cInk, inner sep=0pt] at (3.860,-0.540) {0.50};
\fill[color={rgb,255:red,37;green,106;blue,191}] (4.195,-0.695) rectangle (4.765,-0.385);
\draw[white, dash pattern=on 1.1pt off 0.9pt, line width=0.45pt] (4.207,-0.683) rectangle (4.753,-0.397);
\node[font=\scriptsize, text=white, inner sep=0pt] at (4.480,-0.540) {0.67};
\fill[color={rgb,255:red,131;green,180;blue,239}] (4.815,-0.695) rectangle (5.385,-0.385);
\node[font=\scriptsize, text=cInk, inner sep=0pt] at (5.100,-0.540) {0.26};
\fill[color={rgb,255:red,172;green,205;blue,245}] (5.435,-0.695) rectangle (6.005,-0.385);
\node[font=\scriptsize, text=cInk, inner sep=0pt] at (5.720,-0.540) {0.12};
\fill[color={rgb,255:red,188;green,214;blue,247}] (6.055,-0.695) rectangle (6.625,-0.385);
\node[font=\scriptsize, text=cInk, inner sep=0pt] at (6.340,-0.540) {0.07};
\fill[color={rgb,255:red,54;green,132;blue,226}] (3.575,-1.055) rectangle (4.145,-0.745);
\node[font=\scriptsize, text=white, inner sep=0pt] at (3.860,-0.900) {0.51};
\fill[color={rgb,255:red,143;green,188;blue,241}] (4.195,-1.055) rectangle (4.765,-0.745);
\node[font=\scriptsize, text=cInk, inner sep=0pt] at (4.480,-0.900) {0.22};
\fill[color={rgb,255:red,178;green,208;blue,246}] (4.815,-1.055) rectangle (5.385,-0.745);
\node[font=\scriptsize, text=cInk, inner sep=0pt] at (5.100,-0.900) {0.10};
\fill[color={rgb,255:red,187;green,214;blue,247}] (5.435,-1.055) rectangle (6.005,-0.745);
\node[font=\scriptsize, text=cInk, inner sep=0pt] at (5.720,-0.900) {0.07};
\fill[color={rgb,255:red,184;green,212;blue,246}] (6.055,-1.055) rectangle (6.625,-0.745);
\node[font=\scriptsize, text=cInk, inner sep=0pt] at (6.340,-0.900) {0.08};
\fill[color={rgb,255:red,40;green,114;blue,204}] (3.575,-1.415) rectangle (4.145,-1.105);
\node[font=\scriptsize, text=white, inner sep=0pt] at (3.860,-1.260) {0.62};
\fill[color={rgb,255:red,25;green,81;blue,153}] (4.195,-1.415) rectangle (4.765,-1.105);
\node[font=\scriptsize, text=white, inner sep=0pt] at (4.480,-1.260) {0.82};
\fill[color={rgb,255:red,14;green,57;blue,112}] (4.815,-1.415) rectangle (5.385,-1.105);
\node[font=\scriptsize, text=white, inner sep=0pt] at (5.100,-1.260) {0.98};
\fill[color={rgb,255:red,13;green,54;blue,107}] (5.435,-1.415) rectangle (6.005,-1.105);
\node[font=\scriptsize, text=white, inner sep=0pt] at (5.720,-1.260) {1.00};
\fill[color={rgb,255:red,13;green,54;blue,107}] (6.055,-1.415) rectangle (6.625,-1.105);
\node[font=\scriptsize, text=white, inner sep=0pt] at (6.340,-1.260) {1.00};
\fill[color={rgb,255:red,146;green,189;blue,241}] (3.575,-1.775) rectangle (4.145,-1.465);
\node[font=\scriptsize, text=cInk, inner sep=0pt] at (3.860,-1.620) {0.21};
\fill[color={rgb,255:red,163;green,200;blue,244}] (4.195,-1.775) rectangle (4.765,-1.465);
\node[font=\scriptsize, text=cInk, inner sep=0pt] at (4.480,-1.620) {0.15};
\fill[color={rgb,255:red,184;green,212;blue,246}] (4.815,-1.775) rectangle (5.385,-1.465);
\node[font=\scriptsize, text=cInk, inner sep=0pt] at (5.100,-1.620) {0.08};
\fill[color={rgb,255:red,189;green,215;blue,247}] (5.435,-1.775) rectangle (6.005,-1.465);
\node[font=\scriptsize, text=cInk, inner sep=0pt] at (5.720,-1.620) {0.06};
\fill[color={rgb,255:red,202;green,224;blue,250}] (6.055,-1.775) rectangle (6.625,-1.465);
\node[font=\scriptsize, text=cInk, inner sep=0pt] at (6.340,-1.620) {0.01};
\fill[color={rgb,255:red,96;green,159;blue,233}] (3.575,-2.135) rectangle (4.145,-1.825);
\node[font=\scriptsize, text=cInk, inner sep=0pt] at (3.860,-1.980) {0.38};
\fill[color={rgb,255:red,64;green,139;blue,229}] (4.195,-2.135) rectangle (4.765,-1.825);
\node[font=\scriptsize, text=cInk, inner sep=0pt] at (4.480,-1.980) {0.48};
\fill[color={rgb,255:red,40;green,114;blue,204}] (4.815,-2.135) rectangle (5.385,-1.825);
\node[font=\scriptsize, text=white, inner sep=0pt] at (5.100,-1.980) {0.62};
\fill[color={rgb,255:red,14;green,60;blue,118}] (5.435,-2.135) rectangle (6.005,-1.825);
\node[font=\scriptsize, text=white, inner sep=0pt] at (5.720,-1.980) {0.96};
\fill[color={rgb,255:red,13;green,54;blue,107}] (6.055,-2.135) rectangle (6.625,-1.825);
\node[font=\scriptsize, text=white, inner sep=0pt] at (6.340,-1.980) {1.00};
\node[anchor=north, font=\scriptsize, text=cInk, inner sep=2pt] at (0.310,-2.180) {500};
\node[anchor=north, font=\scriptsize, text=cInk, inner sep=2pt] at (0.930,-2.180) {1k};
\node[anchor=north, font=\scriptsize, text=cInk, inner sep=2pt] at (1.550,-2.180) {2k};
\node[anchor=north, font=\scriptsize, text=cInk, inner sep=2pt] at (2.170,-2.180) {4k};
\node[anchor=north, font=\scriptsize, text=cInk, inner sep=2pt] at (2.790,-2.180) {10k};
\node[anchor=north, font=\scriptsize, text=cInk, inner sep=2pt] at (3.860,-2.180) {500};
\node[anchor=north, font=\scriptsize, text=cInk, inner sep=2pt] at (4.480,-2.180) {1k};
\node[anchor=north, font=\scriptsize, text=cInk, inner sep=2pt] at (5.100,-2.180) {2k};
\node[anchor=north, font=\scriptsize, text=cInk, inner sep=2pt] at (5.720,-2.180) {4k};
\node[anchor=north, font=\scriptsize, text=cInk, inner sep=2pt] at (6.340,-2.180) {10k};
\node[anchor=east, font=\scriptsize, text=cInk, inner sep=1pt] at (-0.120,-2.310) {sample size $n$};
\fill[color={rgb,255:red,203;green,224;blue,250}] (0.025,-2.760) rectangle (0.140,-2.600);
\fill[color={rgb,255:red,198;green,222;blue,250}] (0.135,-2.760) rectangle (0.250,-2.600);
\fill[color={rgb,255:red,194;green,218;blue,248}] (0.245,-2.760) rectangle (0.360,-2.600);
\fill[color={rgb,255:red,190;green,216;blue,248}] (0.355,-2.760) rectangle (0.470,-2.600);
\fill[color={rgb,255:red,185;green,212;blue,246}] (0.465,-2.760) rectangle (0.580,-2.600);
\fill[color={rgb,255:red,180;green,210;blue,246}] (0.575,-2.760) rectangle (0.690,-2.600);
\fill[color={rgb,255:red,176;green,207;blue,245}] (0.685,-2.760) rectangle (0.800,-2.600);
\fill[color={rgb,255:red,170;green,204;blue,245}] (0.795,-2.760) rectangle (0.910,-2.600);
\fill[color={rgb,255:red,166;green,201;blue,245}] (0.905,-2.760) rectangle (1.020,-2.600);
\fill[color={rgb,255:red,160;green,198;blue,244}] (1.015,-2.760) rectangle (1.130,-2.600);
\fill[color={rgb,255:red,156;green,196;blue,244}] (1.125,-2.760) rectangle (1.240,-2.600);
\fill[color={rgb,255:red,151;green,192;blue,242}] (1.235,-2.760) rectangle (1.350,-2.600);
\fill[color={rgb,255:red,146;green,190;blue,242}] (1.345,-2.760) rectangle (1.460,-2.600);
\fill[color={rgb,255:red,141;green,186;blue,240}] (1.455,-2.760) rectangle (1.570,-2.600);
\fill[color={rgb,255:red,136;green,184;blue,240}] (1.565,-2.760) rectangle (1.680,-2.600);
\fill[color={rgb,255:red,132;green,180;blue,239}] (1.675,-2.760) rectangle (1.790,-2.600);
\fill[color={rgb,255:red,126;green,178;blue,238}] (1.785,-2.760) rectangle (1.900,-2.600);
\fill[color={rgb,255:red,122;green,174;blue,238}] (1.895,-2.760) rectangle (2.010,-2.600);
\fill[color={rgb,255:red,116;green,172;blue,237}] (2.005,-2.760) rectangle (2.120,-2.600);
\fill[color={rgb,255:red,112;green,168;blue,236}] (2.115,-2.760) rectangle (2.230,-2.600);
\fill[color={rgb,255:red,107;green,166;blue,236}] (2.225,-2.760) rectangle (2.340,-2.600);
\fill[color={rgb,255:red,102;green,162;blue,234}] (2.335,-2.760) rectangle (2.450,-2.600);
\fill[color={rgb,255:red,97;green,160;blue,234}] (2.445,-2.760) rectangle (2.560,-2.600);
\fill[color={rgb,255:red,92;green,156;blue,232}] (2.555,-2.760) rectangle (2.670,-2.600);
\fill[color={rgb,255:red,87;green,154;blue,232}] (2.665,-2.760) rectangle (2.780,-2.600);
\fill[color={rgb,255:red,82;green,150;blue,231}] (2.775,-2.760) rectangle (2.890,-2.600);
\fill[color={rgb,255:red,77;green,147;blue,230}] (2.885,-2.760) rectangle (3.000,-2.600);
\fill[color={rgb,255:red,71;green,144;blue,230}] (2.995,-2.760) rectangle (3.110,-2.600);
\fill[color={rgb,255:red,65;green,140;blue,230}] (3.105,-2.760) rectangle (3.220,-2.600);
\fill[color={rgb,255:red,60;green,137;blue,229}] (3.215,-2.760) rectangle (3.330,-2.600);
\fill[color={rgb,255:red,56;green,134;blue,228}] (3.325,-2.760) rectangle (3.440,-2.600);
\fill[color={rgb,255:red,52;green,130;blue,224}] (3.435,-2.760) rectangle (3.550,-2.600);
\fill[color={rgb,255:red,49;green,127;blue,221}] (3.545,-2.760) rectangle (3.660,-2.600);
\fill[color={rgb,255:red,46;green,124;blue,218}] (3.655,-2.760) rectangle (3.770,-2.600);
\fill[color={rgb,255:red,44;green,122;blue,216}] (3.765,-2.760) rectangle (3.880,-2.600);
\fill[color={rgb,255:red,42;green,119;blue,212}] (3.875,-2.760) rectangle (3.990,-2.600);
\fill[color={rgb,255:red,40;green,116;blue,207}] (3.985,-2.760) rectangle (4.100,-2.600);
\fill[color={rgb,255:red,40;green,113;blue,202}] (4.095,-2.760) rectangle (4.210,-2.600);
\fill[color={rgb,255:red,39;green,110;blue,198}] (4.205,-2.760) rectangle (4.320,-2.600);
\fill[color={rgb,255:red,38;green,107;blue,193}] (4.315,-2.760) rectangle (4.430,-2.600);
\fill[color={rgb,255:red,36;green,105;blue,189}] (4.425,-2.760) rectangle (4.540,-2.600);
\fill[color={rgb,255:red,34;green,102;blue,185}] (4.535,-2.760) rectangle (4.650,-2.600);
\fill[color={rgb,255:red,32;green,99;blue,181}] (4.645,-2.760) rectangle (4.760,-2.600);
\fill[color={rgb,255:red,31;green,96;blue,177}] (4.755,-2.760) rectangle (4.870,-2.600);
\fill[color={rgb,255:red,29;green,93;blue,173}] (4.865,-2.760) rectangle (4.980,-2.600);
\fill[color={rgb,255:red,28;green,91;blue,169}] (4.975,-2.760) rectangle (5.090,-2.600);
\fill[color={rgb,255:red,27;green,88;blue,164}] (5.085,-2.760) rectangle (5.200,-2.600);
\fill[color={rgb,255:red,26;green,86;blue,160}] (5.195,-2.760) rectangle (5.310,-2.600);
\fill[color={rgb,255:red,25;green,83;blue,156}] (5.305,-2.760) rectangle (5.420,-2.600);
\fill[color={rgb,255:red,24;green,80;blue,151}] (5.415,-2.760) rectangle (5.530,-2.600);
\fill[color={rgb,255:red,23;green,78;blue,147}] (5.525,-2.760) rectangle (5.640,-2.600);
\fill[color={rgb,255:red,22;green,75;blue,143}] (5.635,-2.760) rectangle (5.750,-2.600);
\fill[color={rgb,255:red,20;green,72;blue,139}] (5.745,-2.760) rectangle (5.860,-2.600);
\fill[color={rgb,255:red,18;green,70;blue,135}] (5.855,-2.760) rectangle (5.970,-2.600);
\fill[color={rgb,255:red,17;green,67;blue,131}] (5.965,-2.760) rectangle (6.080,-2.600);
\fill[color={rgb,255:red,16;green,65;blue,127}] (6.075,-2.760) rectangle (6.190,-2.600);
\fill[color={rgb,255:red,15;green,62;blue,122}] (6.185,-2.760) rectangle (6.300,-2.600);
\fill[color={rgb,255:red,14;green,60;blue,118}] (6.295,-2.760) rectangle (6.410,-2.600);
\fill[color={rgb,255:red,14;green,58;blue,114}] (6.405,-2.760) rectangle (6.520,-2.600);
\fill[color={rgb,255:red,13;green,55;blue,109}] (6.515,-2.760) rectangle (6.630,-2.600);
\node[anchor=east, font=\scriptsize, text=cInk, inner sep=1pt] at (-0.120,-2.680) {rejection rate under $H_0$};
\draw[cAxis, line width=0.4pt] (0.025,-2.760) -- (0.025,-2.830);
\node[anchor=north, font=\scriptsize, text=cInk, inner sep=1pt] at (0.025,-2.830) {0.00};
\draw[cAxis, line width=0.4pt] (1.675,-2.760) -- (1.675,-2.830);
\node[anchor=north, font=\scriptsize, text=cInk, inner sep=1pt] at (1.675,-2.830) {0.25};
\draw[cAxis, line width=0.4pt] (3.325,-2.760) -- (3.325,-2.830);
\node[anchor=north, font=\scriptsize, text=cInk, inner sep=1pt] at (3.325,-2.830) {0.50};
\draw[cAxis, line width=0.4pt] (4.975,-2.760) -- (4.975,-2.830);
\node[anchor=north, font=\scriptsize, text=cInk, inner sep=1pt] at (4.975,-2.830) {0.75};
\draw[cAxis, line width=0.4pt] (6.625,-2.760) -- (6.625,-2.830);
\node[anchor=north, font=\scriptsize, text=cInk, inner sep=1pt] at (6.625,-2.830) {1.00};
\end{tikzpicture}
 }
    \caption{
        Rejection rate under $H_0$ against $n$ in Experiment~1 (left, $\psi = \psi_5$) and Experiment~2 (right, $\psi = E_1$) for the residual correlation test (standard) and the robust test, run at $\lambda = 1$, with the embeddings and tolerances named in the rows. Dashed cells are fitted by ridge regression, since $n$ is below the dimension of $\psi(Z)$.
    }
    \label{fig:exp123}
\end{figure}
    
\subsection{Gaussian data}\label{sec:exp1} \textbf{Design.} We draw $Z \sim N(0,
I_{p})$ with $p = \resOneDimZ$ and set $X = \theta_X^\top Z + \varepsilon_X$ and
$Y = \theta_Y^\top Z + \varepsilon_Y$ with independent standard normal noise, so
that $H_0$ holds. Of the embeddings $\psi_q(Z) = (Z_1, \dots, Z_q)$, $\psi_{p}$
is sufficient and $\psi_{\resOneQinsuff}$ is not. The coefficients of the
dropped coordinates set $R^2_{\psi_q}(X) = R^2_{\psi_q}(Y) = \resOneRtwoA$, so
the true insufficiency is $\sqrt{R^2_{\psi_q}(X) R^2_{\psi_q}(Y)} =
\resOneBound$. The correlation $\lambda$ of the discarded parts is the cosine of
the angle between the entries of $\theta_X$ and of $\theta_Y$ on the dropped
coordinates. We take them orthogonal ($\lambda = 0$) or parallel ($\lambda =
1$). We record the rejection rate at varying $n$
under $H_0$ of the residual correlation test with the sufficient embedding and
with the insufficient embedding for both $\lambda$, and of the robust test at
$\lambda = 1$ with the tolerance equal to the true insufficiency and to half of
it, the tolerance of an analyst who wrongly assumes $\lambda \le 1/2$.

\textbf{Results.} \Cref{fig:exp123} (left) shows the rejection rates. With the
sufficient embedding, and with the insufficient embedding at $\lambda = 0$, the
residual correlation test is approximately valid for every $n$; at $\lambda = 1$ its rejection rate rises to one. The robust test holds its level when
the tolerance is set to the true insufficiency, but its rejection rate is inflated when the tolerance is set to only half of the true insufficiency.
\Cref{fig:exp1a,fig:exp1b} in \cref{app:figs} add the tightness of the bound,
power, and the distribution of $\hat\rho_\psi$ and $\RV_\alpha$.

\subsection{Real text with synthetic outcomes}\label{sec:exp2} 
\textbf{Design.}
We repeat the Gaussian experiment with a real $Z$: \resTwoN\ posts from the
20~Newsgroups corpus \citep{lang1995}, encoded by two sentence encoders from the
Sentence-BERT family \citep{reimers2019}, $E_1$ with $d = \resTwoDimWeak$ and
$E_2$ with $d = \resTwoDimStrong$. We generate Gaussian $X$ and $Y$ from the
embeddings of $Z$ so that $E_1$ is insufficient with $\sqrt{R^2_{E_1}(X)
R^2_{E_1}(Y)} = \resTwoPopBound$ and a correlation $\lambda$ of the discarded
parts set by design, while the four generating features and the
\resTwoRichDim-dimensional $[E_1, E_2]$ are sufficient
(\cref{app:impl-exp2}).

\textbf{Results.} \Cref{fig:exp123} (right) shows the rejection rates under $H_0$ for $n$ from \resTwoNmin\ to \resTwoNmax. 
They match Experiment~1: with $\psi = E_1$, the residual correlation test stays close to its level at $\lambda = 0$ and rejects with certainty at $\lambda = 1$, where $\hat\rho_\psi$ sits at the true insufficiency, and the robust test holds its level with the true insufficiency as its tolerance and not with half of it. The sufficient sets stay close to their level.
\Cref{fig:exp2,fig:exp2power} in \cref{app:figs} report $\hat\rho_\psi$ and $\RV_\alpha$, and the power at
$n = \resTwoPowerN$.

\begin{figure*}[t]
\centering
\figorbox{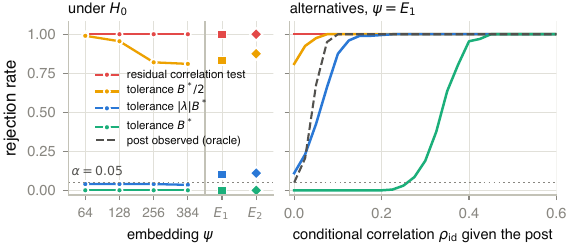}{0.57\textwidth}\hfill
\figorbox{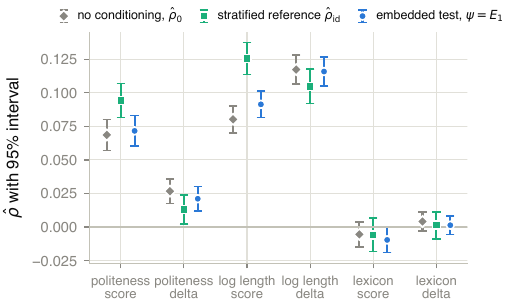}{0.41\textwidth}
\caption{
    (\textit{Left}): Experiment~3, rejection rate under $H_0$ over \resFourTruthR\ null datasets drawn from the replicates, for the residual correlation test and for the robust test with the tolerance equal to half the true insufficiency, to the true bias $|\lambda|\sqrt{R^2_\psi(X) R^2_\psi(Y)}$ and to the true insufficiency ($B^*$ in the legend); $\psi(Z)$ is 64 to 384 principal components of the generator's mean hidden state over the prompt tokens at its last layer, $E_1$ or $E_2$. (\textit{Middle}): The same tests with $\psi = E_1$ against a range of alternatives. We plot a reference test (dashed) that conditions on the post through the unused draws. 
    (\textit{Right}): Reddit data application. For each pair of $X$ (politeness rating, log length, politeness lexicon) and $Y$ (signed log score, delta indicator) we plot the unconditional correlation $\hat\rho_0$ (diamonds), $\hat\rho_{\id}$ from the baseline, which conditions on the comment being replied to (squares), and $\hat\rho_\psi$ from the embedded test with $\psi = E_1$ (circles), all with 95\% intervals. 
    }
\label{fig:exp4app}
\end{figure*}
\subsection{Text generated by a local model with accessible states}\label{sec:exp4}
\textbf{Design.} 
Using a language model, we generate text from real $Z$. We take \resThreeN\
posts $Z_i$ from the corpus of \cref{sec:exp2} and, in independent calls to
\resFourModel\ \citep{qwen2025}, a \resFourParams-parameter open model, generate
a reader's comment $\tilde X$ on the post, written from a randomly assigned
persona that acts as independent noise, and a newsletter summary $\tilde Y$ of
the post (prompts in \cref{app:prompts}), so that $H_0$ holds exactly. The comment ends with a line labelling its own sentiment as positive, neutral or negative, which is removed before the comment is used. The tested
variable $X$ is a sentiment score of the comment, namely the cross-fitted
prediction of its label from the $E_1$ embedding of the comment, and $Y$ applies
the same fitted direction to the embedding of the summary. We compare $\psi$
equal to $E_1$ and to $E_2$, the two encoders of \cref{sec:exp2}, with the
generator's internal states. From its forward pass over the post we keep the
mean of the hidden states over the prompt tokens at the last of its
\resFourLayers\ layers. For a $d$-dimensional embedding, we
take the first $d$ principal components of this mean, and we compare it with
$E_1$ and $E_2$ at the same dimension.

Since we have access to the generating model, we can get ground-truth values of
the sensitivity parameters. For every post we draw the comment and the summary
\resFourTruthM\ more times from the same model, prompts and sampling settings,
with the persona redrawn each time. By \cref{prop:replicates}, these draws identify $R^2_\psi(X)$, $R^2_\psi(Y)$ and $\lambda$ for every embedding, hence
the ground-truth insufficiency $\sqrt{R^2_\psi(X) R^2_\psi(Y)}$ and the bias
$|\lambda|\sqrt{R^2_\psi(X) R^2_\psi(Y)}$ of \cref{thm:pair}. Taking one draw of
each text per post gives an exact null dataset; on \resFourTruthR\ such datasets
we record the rejection rates of the residual correlation test and of the robust
test with the tolerance equal to the ground-truth insufficiency, to half of it,
and to the ground-truth bias. We can also use the replicates to evaluate power
under the alternative: within each post we pair the comment draw, with
probability $\pi$, with the summary draw of the same sentiment rank, and
otherwise with an independent draw. The probability $\pi$ controls the conditional correlation
$\rho_{\id}$ of $X$ and $Y$ given the post. Details are in \cref{app:impl-exp3}.

\textbf{Results.} 
Under $H_0$, the residual correlation test (red in \cref{fig:exp4app}, left) rejects with every embedding. Against an unadjusted correlation of \resFourMargNull, the encoders leave $\hat\rho_\psi = \resFourRhoEone$ with $E_1$ and \resFourRhoEtwo\ with $E_2$. The generating model's own states do not remove the bias either: at every dimension, from 64 to 768 principal components, they leave a residual correlation of the same size (\cref{fig:exp4views} in \cref{app:figs}). Two explanations are consistent with this: the aggregated states are less informative than attention to every token, or the regressions do not recover $\E[X \mid \psi(Z)]$ at this $n$. 

The discarded parts are correlated, with $\lambda$ between
\resFourTruthLambdaMinStates\ and \resFourTruthLambdaMaxStates\ across
embeddings. \Cref{fig:exp4app} (left) shows the rejection rates under the null. Using a tolerance equal to the true bias (blue) gives a robust test that approximately holds the level $\alpha = 0.05$. If we use the true insufficiency
and assume $\lambda = 1$, the robust test (green) is conservative and never
rejects. With half the true insufficiency (orange) it rejects almost as often as the residual correlation test. The encoders $E_1$ and $E_2$ behave similarly to the
generator's own states. \Cref{fig:exp4app} (middle) shows the corresponding
power. Using the replicates, we plot a reference test that does not use an embedding but conditions on the post itself (dashed). The residual correlation
test and the robust test with tolerance $B^*/2$ reject throughout, but since they do not hold
their level, their power means nothing. The robust test with the true bias as
tolerance detects almost the same alternatives as the reference test, while the
conservative test using the true insufficiency as tolerance only starts
rejecting at the strongest alternatives. 
In conclusion, this experiment shows that a plain residual correlation test can
be biased by an insufficient embedding, while the robust test with correct tolerance approximately holds its level and has power, but underestimating the
tolerance leads to a biased test and overestimating it to a loss of power.

\subsection{Application: style and success of replies in online discussions}\label{sec:app}
\textbf{Design.} We return to the conceptual example of \cref{sec:intro} with the ChangeMyView corpus \citep{tan2016,chang2020convokit}: \resAppNops\ posts with their discussion threads, the Reddit score of every comment, and the `deltas' that the original poster awards to comments that changed their view. Every observation is a reply to a comment, and $Z$ is the comment being replied to, with \resDeepN\ replies to \resDeepNparents\ comments (\cref{app:impl-app}). 
$X$ is either the politeness of the reply, rated by a model on a five-point scale, or its log length. $Y$ is the signed log score of the reply or its delta indicator.
Because several replies share a $Z$, we can construct a baseline test that conditions on $Z$ without a representation of it: under $H_0$, the values of $X$ among the replies to a comment carry no information about their values of $Y$, so the covariance of $X$ and $Y$ after demeaning within comments estimates $\E[\Cov(X, Y \mid Z)]$, and permuting $X$ among the replies to a comment gives an exact test under exchangeability. We compare the baseline to a test that conditions on an embedding $E_1(Z)$ of the comment, and also plot the unconditional correlation $\hat\rho_0$.

\textbf{Results.}
The embedded $\hat\rho_\psi$ stays close to the unconditional value, because the
embeddings do not fully recover the effect of the comment (\cref{fig:exp4app},
right, and \cref{fig:appviews} in \cref{app:figs}). The replies to a comment are
replicates in the sense of \cref{prop:replicates}, so we can recover the
sensitivity parameters from variance components: the comment-level components of
$X$ and $Y$ are correlated in the politeness and length pairs, with
$\hat\lambda$ between $\resDeepLambdaCompMin$ and $\resDeepLambdaCompMax$.
 \section{Discussion}\label{sec:discussion}

We have highlighted the bias arising in CI tests from using embeddings of the
conditioning variable. 
The bias of such a test is described by both the insufficiency of the embedding and the correlation of the discarded parts. 
We propose a robust test that controls the level if the insufficiency tolerance is correctly specified. 

Our results are largely negative about testing conditional independence given an embedding. The bias bound is worst-case in $\lambda$, and the
tolerance is in general not identified from the data. In our experiments, richer
embeddings demonstrated almost none of the insufficiency, in part because the
estimated benchmark loses the excess estimation error of the richer regression (\cref{app:impl-bench}). Even the generator's own states were insufficient, and
only replicates identified the sensitivity parameters. Without a tolerance
justified by domain knowledge, a rejection is therefore hard to distinguish from
an artifact of the embedding. Future work could investigate how to adapt the
embeddings to retain more information about the tested variables, to ensure the validity of the tests, or how to construct better benchmarks for the insufficiency.

\subsection*{AI use statement}
Generative AI tools were used for drafting and editing text and proofs and for writing and testing the experiment code. In Experiment~3 (\cref{sec:exp4}), we used a language model to generate the synthetic texts that constitute the data; we used a decision model in the application (\cref{sec:app}) to rate the politeness of the replies, and a language model to rate a subset of the replies as a check (\cref{app:impl-app}); the prompts are reported in \cref{app:prompts,app:impl-app}. We have reviewed all AI-assisted work, and all mathematical statements and proofs were checked by the authors, who take responsibility for the content of this work.

\acks{Nikolaj Thams is supported by the Novo Nordisk Foundation (grant NNF26OC0114284). The regenerated texts of Experiment~3 were computed on the Gefion AI supercomputer, operated by the Danish Centre for AI Innovation (DCAI).}

\clearpage
\appendix
\crefalias{section}{appendix}
\crefalias{subsection}{subappendix}
\numberwithin{equation}{section} \numberwithin{figure}{section} \numberwithin{table}{section}
\section{Proofs}\label{app:proofs}
In this section, write $V := \psi(Z)$, which takes values in $\mathcal V$.

\subsection{Proofs for \texorpdfstring{\cref{sec:setup}}{Section 2}}\label{app:setup}
 Since $\psi$ is measurable, $\sigma(V) \subseteq \sigma(Z)$, and $\mathcal F_\psi$ consists of the $f \in \mathcal F$ for which $f(X, Z)$ is a function of $(X, V)$. If $\sigma_\psi(f) = 0$, then $f = \E[f \mid V]$ almost surely, so $\E[\Cov(f, g \mid V)] = 0$ for every $g$ and $\Delta_\psi(f) = 0$. The conventions $\rho_\psi(f, g) := 0$ and $R^2_\psi(f) := 0$ thus replace $0/0$ by $0$, and every pair satisfies $\E[\Cov(f, g \mid V)] = 0$ whenever $\rho_\psi(f, g) = 0$.

\begin{proof}[Proof of the claims in \cref{example:family}]
(a) Here $\mathcal F_\psi$ and $\mathcal G_\psi$ consist of all $f \in L^2(P_{X,Z})$ for which $f(X, Z)$ is a function of $(X, V)$ and all $g \in L^2(P_{Y,Z})$ for which $g(Y, Z)$ is a function of $(Y, V)$. If $X \indep Y \mid V$, then $(X, V) \indep (Y, V) \mid V$, so $\E[fg \mid V] = \E[f \mid V]\, \E[g \mid V]$ for all $f \in \mathcal F_\psi$ and $g \in \mathcal G_\psi$, every $\rho_\psi(f, g)$ is zero, and so is the supremum. Conversely, if $\rho_\psi(\mathcal F, \mathcal G) = 0$, then every $\rho_\psi(f, g)$ is zero, so $\E[\Cov(f, g \mid V)] = 0$ for all $f \in \mathcal F_\psi$ and $g \in \mathcal G_\psi$. For measurable sets $A$, $B$, and $C$, take $f = 1_A(X)$ and $g = 1_B(Y) 1_C(V)$. Then $0 = \E[\Cov(f, g \mid V)] = \E\big[1_C(V) \Cov(1_A(X), 1_B(Y) \mid V)\big]$ for all $C$, so $\Cov(1_A(X), 1_B(Y) \mid V) = 0$ almost surely, that is, $P(X \in A, Y \in B \mid V) = P(X \in A \mid V)\, P(Y \in B \mid V)$ almost surely for all $A$ and $B$. This is $X \indep Y \mid V$.

(b) Let $f \in \mathcal F_\psi$. Since $f - \E[f \mid V]$ is a function of $(X, V)$, we have by the tower property and the Cauchy--Schwarz inequality that
\begin{align*}
\E[\Cov(f, Y \mid V)]^2 &= \E[(f - \E[f \mid V])(Y - \E[Y \mid V])]^2 \\
&= \E[(f - \E[f \mid V])(\E[Y \mid X, V] - \E[Y \mid V])]^2\\
&\leq \sigma^2_\psi(f) \E\left[ (\E[Y \mid X, V] - \E[Y \mid V])^2 \right].
\end{align*}
Therefore, for any $f \in \mathcal{F}_\psi$, we have 
\[
  \rho_{\psi}(f, (y,z) \mapsto y) \leq \frac{ \E\left[ (\E[Y \mid X, V] - \E[Y \mid V])^2 \right]^{1/2}}{\sigma_\psi(Y)}.
\]
The function $f(x, z) = \E[Y \mid X = x, V = \psi(z)]$ is a square-integrable
function of $(x, \psi(z))$, so it lies in $\mathcal F_\psi$, and it attains the
upper bound, which is therefore the value of $\rho_\psi(\mathcal F, \mathcal
G)$. From this, we see that $\rho_\psi(\mathcal{F}, \mathcal{G}) = 0$ if and
only if $\E[Y \mid X, V] = \E[Y \mid V]$. The argument assumes $\sigma_\psi(Y) >
0$. If $\sigma_\psi(Y) = 0$, then $Y = \E[Y \mid V]$, so $\E[Y \mid X, V] = \E[Y
\mid V]$, and the measure is zero by convention.

(c) The classes are singletons, and their elements do not depend on $z$, so $\mathcal F_\psi = \mathcal F$ and $\mathcal G_\psi = \mathcal G$.
\end{proof}

\begin{proof}[Proof of \cref{thm:pair}]
If $\sigma_\psi(f) = 0$ or $\sigma_\psi(g) = 0$, both sides of the bound are zero by convention, so let $\sigma_\psi(f), \sigma_\psi(g) > 0$. We have $f - \E[f \mid V] = (f - \E[f \mid Z]) + \Delta_\psi(f)$, and likewise for $g$. Since $f - \E[f \mid Z]$ has conditional mean zero given $Z$ and $\Delta_\psi(g)$ is a function of $Z$, the cross terms have expectation zero, so
\begin{equation}\label{eq:pair-fg}
\E[\Cov(f, g \mid V)] = \E[\Cov(f, g \mid Z)] + \E[\Delta_\psi(f)\, \Delta_\psi(g)].
\end{equation}
Since the discarded parts have mean zero,
\[
\E[\Delta_\psi(f)\Delta_\psi(g)] = \lambda_\psi(f, g) \sqrt{R^2_\psi(f) R^2_\psi(g)}\, \sigma_\psi(f) \sigma_\psi(g).
\]
Dividing by $\sigma_\psi(f)\sigma_\psi(g)$ and using $\sigma_{\id}^2(f) = (1 - R^2_\psi(f))\, \sigma_\psi^2(f)$, and likewise for $g$, we have
\[
\begin{split}
\frac{\E[\Cov(f, g \mid V)]}{\sigma_\psi(f)\, \sigma_\psi(g)} ={}& \sqrt{(1 - R^2_\psi(f))(1 - R^2_\psi(g))}\; \frac{\E[\Cov(f, g \mid Z)]}{\sigma_{\id}(f)\, \sigma_{\id}(g)} + \lambda_\psi(f, g) \sqrt{R^2_\psi(f)\, R^2_\psi(g)},
\end{split}
\]
where the second fraction is read as zero if $\sigma_{\id}(f)\sigma_{\id}(g) = 0$, in which case $\E[\Cov(f, g \mid Z)] = 0$. The two fractions have absolute values $\rho_\psi(f, g)$ and $\rho_{\id}(f, g)$, so the reverse triangle inequality gives the bound. Under $H_0$, $\E[\Cov(f, g \mid Z)] = 0$, so $\rho_\psi(f, g) = |\lambda_\psi(f, g)| \sqrt{R^2_\psi(f)\, R^2_\psi(g)}$.
\end{proof}

\begin{proof}[Proof of \cref{cor:sup}]
If $\mathcal F_\psi$ or $\mathcal G_\psi$ is empty, all quantities are zero, so let both be nonempty.

(i) Let $f \in \mathcal F_\psi$ and $g \in \mathcal G_\psi$. Since the factor multiplying $\rho_{\id}(f, g)$ in \cref{thm:pair} is at most one, we have
\[
\rho_\psi(f, g) \le \rho_{\id}(f, g) + |\lambda_\psi(f, g)| \sqrt{R^2_\psi(f)\, R^2_\psi(g)} \le \rho_{\id}(\mathcal F, \mathcal G) + \sqrt{R^2_\psi(\mathcal F)\, R^2_\psi(\mathcal G)},
\]
where the second inequality uses $|\lambda_\psi(f, g)| \le 1$, $R^2_\psi(f) \le R^2_\psi(\mathcal F)$, $R^2_\psi(g) \le R^2_\psi(\mathcal G)$, and that $(f, g) \in \mathcal F \times \mathcal G$, because $\mathcal F_\psi \subseteq \mathcal F$ and $\mathcal G_\psi \subseteq \mathcal G$. Taking the supremum over $f$ and $g$ gives (i).

(ii) Under $H_0$, \cref{thm:pair} gives $\rho_\psi(f, g) = |\lambda_\psi(f, g)| \sqrt{R^2_\psi(f)\, R^2_\psi(g)}$ for all $f \in \mathcal F_\psi$ and $g \in \mathcal G_\psi$, and taking the supremum gives the equality in (ii). The supremum is zero if and only if every term is zero. If $R^2_\psi(f) = 0$ or $R^2_\psi(g) = 0$, the corresponding discarded part vanishes, and $\lambda_\psi(f, g) = 0$ by convention. Hence every term is zero if and only if $\lambda_\psi(f, g) = 0$ for all $f$ and $g$.
\end{proof}

\begin{lemma}[Zero partial $R^2$]\label[lemma]{lem:zeror2}
\begin{enumerate}[label=(\roman*),leftmargin=*,itemsep=1pt,topsep=2pt]
\item $R^2_\psi(L^2(P_{X,Z})) = 0$ if and only if $X \indep Z \mid \psi(Z)$.
\item For real-valued, square-integrable $X$, $R^2_\psi(\{(x, z) \mapsto x\}) = 0$ if and only if $\E[X \mid Z] = \E[X \mid \psi(Z)]$.
\end{enumerate}
\end{lemma}
\begin{proof}
For $f \in \mathcal F_\psi$, $R^2_\psi(f) = 0$ if and only if $\Delta_\psi(f) =
0$, that is, $\E[f \mid Z] = \E[f \mid V]$; this holds by definition if
$\sigma_\psi(f) > 0$ and by convention otherwise. (i) For $\mathcal F =
L^2(P_{X,Z})$, $\mathcal F_\psi$ consists of the $f \in L^2(P_{X,Z})$ for which
$f(X, Z)$ is a function of $(X, V)$. If $X \indep Z \mid V$, then $(X, V) \indep
Z \mid V$ because $V$ is a function of $Z$, so $\E[f \mid Z] = \E[f \mid V]$ for
every bounded $f \in \mathcal F_\psi$, and by truncation for every
square-integrable one. Conversely, if $R^2_\psi(\mathcal F) = 0$, taking $f =
1_A(X)$ gives $P(X \in A \mid Z) = P(X \in A \mid V)$ for every measurable $A$,
which is $X \indep Z \mid V$. (ii) For $\mathcal F = \{(x, z) \mapsto x\}$, the
condition is $\E[X \mid Z] = \E[X \mid V]$.
\end{proof}

\begin{proof}[Computations for \cref{ex:hetero}]
Given $Z$, $X$ and $Y$ are independent with means $Z_1 + a Z_2$ and $Z_1 +
b(\lambda Z_2 + \sqrt{1 - \lambda^2}\, Z_3)$ and common variance $Z_4^2$, so
$H_0$ holds, while given $Z_1$ both have mean $Z_1$. Hence $\Delta_\psi(X) = a
Z_2$, $\Delta_\psi(Y) = b(\lambda Z_2 + \sqrt{1 - \lambda^2}\, Z_3)$,
$\sigma_\psi^2(X) = a^2 + 1$, and $R^2_\psi(X) = a^2/(1 + a^2)$, and likewise
for $Y$. For $a, b > 0$, the discarded parts have correlation $\lambda$, and
\cref{thm:pair} gives the measure in \cref{example:family}(c) under $H_0$. For
the second moments, write $\tilde X := X - Z_1 = a Z_2 + Z_4 \varepsilon_X$ and $\tilde Y := Y
- Z_1 = b(\lambda Z_2 + \sqrt{1 - \lambda^2}\, Z_3) + Z_4 \varepsilon_Y$, which
are independent of $Z_1$. Expanding $\tilde X^2 \tilde Y^2$, every term with a single factor
$\varepsilon_X$ or $\varepsilon_Y$ has mean zero, so
\begin{align*}
\E[\tilde X^2 \tilde Y^2] &= a^2 b^2\, \E\big[Z_2^2 (\lambda Z_2 + \sqrt{1 - \lambda^2}\, Z_3)^2\big] + a^2 + b^2 + \E[Z_4^4] \\ &= a^2 b^2 (1 + 2\lambda^2) + a^2 + b^2 + 3,
\end{align*}
while $\E[\tilde X^2]\, \E[\tilde Y^2] = (a^2 + 1)(b^2 + 1)$. Hence $\Cov((X - Z_1)^2, (Y - Z_1)^2 \mid Z_1) = \Cov(\tilde X^2, \tilde Y^2) = 2(1 + a^2 b^2 \lambda^2)$. Thus $f(x, z) = (x - z_1)^2$ and $g(y, z) = (y - z_1)^2$, which are functions of $(x, \psi(z))$ and of $(y, \psi(z))$, have $\rho_\psi(f, g) > 0$, and $H_0^\psi$ fails by \cref{example:family}(a).
\end{proof}

\subsection{Coarsening under alternatives}\label{app:coarse}

For $f \in \mathcal F_\psi$ and $g \in \mathcal G_\psi$, the law of total covariance gives
\begin{equation}\label{eq:local}
\Cov(f, g \mid V) = \E\big[\Cov(f, g \mid Z) \,\big|\, V\big] + \E\big[\Delta_\psi(f)\, \Delta_\psi(g) \,\big|\, V\big],
\end{equation}
since $\E[f \mid Z] - \E[f \mid V] = \Delta_\psi(f)$ has conditional mean zero
given $V$. The second term is created by the discarded information, and the
first is the dependence given $Z$, averaged over the values of $Z$ that $\psi$
maps to the same point. For classes of functions of $x$ and of $y$ alone, such
as those of \cref{example:family}(c), only the overall average of this
dependence enters the measure, and coarsening preserves it. Classes whose
transformations depend on $z$ weight the dependence given $Z$. After coarsening,
they can only weight its averages over the merged values, and they lose
dependence that changes sign within them.

\begin{example}[An embedding that misses an effect modifier]\label[example]{ex:modifier}
Let $Z = (Z_1, Z_2)$ have independent coordinates, $Z_1$ standard normal and $Z_2$ uniform on $\{-1, 1\}$, let $\psi(Z) = Z_1$, and let
\[
X = Z_1 + \varepsilon_X, \qquad Y = Z_1 + Z_2\, \varepsilon_X + \varepsilon_Y
\]
with independent standard normal noise. Given $Z$, the residual of $X$ enters
$Y$ with slope $Z_2$, so $\Cov(X, Y \mid Z) = Z_2$ changes sign with $Z_2$. The
measure of \cref{example:family}(c) averages it out and is zero, whereas the
measure of \cref{example:family}(b) is $1/\sqrt 2$, attained by $f(x, z) = z_2
(x - z_1)$. Given $\psi(Z)$, the slope averages to zero, so $\E[Y \mid X,
\psi(Z)] = \E[Y \mid \psi(Z)]$ and both measures vanish, although $H_0^\psi$
fails. For the classes of (b), $R^2_\psi(\mathcal F) = R^2_\psi(\mathcal G) =
0$: the measure of (b) loses the alternative only because $\psi$ merges the
values of $Z_2$ that reverse the dependence, whereas the measure of (c) had
nothing to lose.
\end{example}

\begin{proof}[Computations for \cref{ex:modifier}]
Given $Z$, the residuals are $\varepsilon_X$ and $Z_2 \varepsilon_X +
\varepsilon_Y$, with covariance $Z_2$ and variances $1$ and $2$, so the measure
of \cref{example:family}(c) is $|\E[Z_2]|/\sqrt 2 = 0$. For $f \in
L^2(P_{X,Z})$, $\E[\Cov(f, Y \mid Z)] = \E[(f - \E[f \mid Z])\, Z_2
\varepsilon_X]$ because $\varepsilon_Y$ is independent of $(X, Z)$, so by the
Cauchy--Schwarz inequality the measure of \cref{example:family}(b) is $\E[(Z_2
\varepsilon_X)^2]^{1/2}/\sigma_{\id}(Y) = 1/\sqrt 2$, attained by $f(x, z) =
z_2(x - z_1)$. Since $Z_2$ is independent of $(Z_1, \varepsilon_X,
\varepsilon_Y)$, $\E[Y \mid X, Z_1] = Z_1 + \E[Z_2]\, \varepsilon_X = Z_1 = \E[Y
\mid Z_1]$, so $\E[\Cov(f, Y \mid Z_1)] = 0$ for every $f$ that is a function of
$(X, Z_1)$, and both measures given $\psi(Z)$ vanish. On the other hand, $\Var(Y
\mid X, Z_1) = (X - Z_1)^2 + 1$ depends on $X$, so $H_0^\psi$ fails. Finally,
$(X, Z_1)$ is independent of $Z_2$, so $\E[f \mid Z] = \E[f \mid Z_1]$ for every
such $f$, and $\E[Y \mid Z] = Z_1 = \E[Y \mid Z_1]$, so $R^2_\psi(\mathcal F) =
R^2_\psi(\mathcal G) = 0$.
\end{proof}

Even the measure of \cref{example:family}(a) can lose all dependence. Let $X$
and $Z$ be independent fair bits, $Y := X \oplus Z$ (exclusive or), and $\psi$
constant. Again $R^2_\psi(\mathcal F) = R^2_\psi(\mathcal G) = 0$, since $X
\indep Z$ and $Y \indep Z$, and $\rho_\psi(\mathcal F, \mathcal G) = 0$ because
$X \indep Y$. Yet $\rho_{\id}(\mathcal F, \mathcal G) = 1$, attained by $f(x, z)
= x$ and $g(y, z) = y(1 - 2z)$, which have the same residual $X - 1/2$ given
$Z$.

\subsection{Proof of \texorpdfstring{\cref{prop:normal}}{Proposition \ref{prop:normal}}}\label{app:normal}

The sample version of the asymptotic variance $v_\psi^2$ of \cref{sec:gcm} replaces $\xi_X$ and $\xi_Y$ by $\hat\xi_{X,i} := \hat r_{X,i}/\hat\sigma_\psi(X)$ and $\hat\xi_{Y,i} := \hat r_{Y,i}/\hat\sigma_\psi(Y)$, $\rho^\pm_\psi$ by $\hat\rho^\pm_\psi$, and the variance by an average,
\begin{equation}\label{eq:vhat}
\hat v_\psi^2 := \frac1n \sum_{i=1}^n \Big(\hat\xi_{X,i} \hat\xi_{Y,i} - \frac{\hat\rho^\pm_\psi}{2}\big(\hat\xi_{X,i}^2 + \hat\xi_{Y,i}^2\big)\Big)^2;
\end{equation}
the summands have average zero, since $\frac1n\sum_i \hat\xi_{X,i} \hat\xi_{Y,i} = \hat\rho^\pm_\psi$ and $\frac1n\sum_i \hat\xi_{X,i}^2 = \frac1n\sum_i \hat\xi_{Y,i}^2 = 1$.

\begin{proof}
The number of folds $K$ is fixed, and $I_k$ denotes the set of indices in fold $k$; for $i \in I_k$, $\hat m_X(\psi(Z_i))$ is the prediction at $\psi(Z_i)$ of a regression fitted on the observations outside fold $k$, and likewise for $Y$. Write
\[
r_{X,i} := X_i - \E[X_i \mid \psi(Z_i)], \qquad e_{X,i} := \hat m_X(\psi(Z_i)) - \E[X_i \mid \psi(Z_i)], \qquad M_X := \frac1n\sum_{i=1}^n e_{X,i}^2,
\]
so that $\hat r_{X,i} = r_{X,i} - e_{X,i}$ and, by \cref{ass:product}, $M_X = o_P(n^{-1/2})$; define $r_{Y,i}$, $e_{Y,i}$, and $M_Y$ likewise. Then $\xi_{X,i} = r_{X,i}/\sigma_\psi(X)$ and $\xi_{Y,i} = r_{Y,i}/\sigma_\psi(Y)$ are the standardized residuals of observation $i$, $\sigma_\psi^2(X) = \E[r_{X,i}^2]$, and $\kappa_\psi := \E[\Cov(X, Y \mid \psi(Z))] = \E[r_{X,i} r_{Y,i}]$, so that $\rho^\pm_\psi = \kappa_\psi/(\sigma_\psi(X)\sigma_\psi(Y))$. By \eqref{eq:rhohat}, $\hat\rho^\pm_\psi = \hat\kappa_\psi/(\hat\sigma_\psi(X)\hat\sigma_\psi(Y))$ with $\hat\kappa_\psi := \frac1n\sum_i R_i$.

We first show two bounds on the estimation errors. Since $\sum_i a_i^2 \le (\sum_i a_i)^2$ for $a_i \ge 0$,
\begin{equation}\label{eq:fourth}
\frac1n\sum_{i=1}^n e_{X,i}^4 \le n M_X^2 = o_P(1),
\end{equation}
and likewise for $Y$. Next, we show that
\begin{equation}\label{eq:cross}
\frac{1}{\sqrt n}\sum_{i=1}^n r_{X,i} e_{Y,i} = o_P(1),
\end{equation}
and that the same holds with $r_{Y,i} e_{X,i}$, $r_{X,i} e_{X,i}$, or $r_{Y,i} e_{Y,i}$ in place of $r_{X,i} e_{Y,i}$; the four proofs are identical. Fix $k$, let $S_k := n^{-1/2}\sum_{i \in I_k} r_{X,i} e_{Y,i}$, and let $\mathcal G_k$ be the $\sigma$-algebra generated by the observations outside fold $k$ and by $\psi(Z_i)$, $i \in I_k$. Each $e_{Y,i}$ with $i \in I_k$ is $\mathcal G_k$-measurable, and since the observations are independent, the $r_{X,i}$ with $i \in I_k$ are conditionally independent given $\mathcal G_k$, with $\E[r_{X,i} \mid \mathcal G_k] = \E[r_{X,i} \mid \psi(Z_i)] = 0$. Hence $\E[S_k \mid \mathcal G_k] = 0$ and, by the Cauchy--Schwarz inequality,
\[
\E[S_k^2 \mid \mathcal G_k] = \frac1n\sum_{i \in I_k} \E[r_{X,i}^2 \mid \psi(Z_i)]\, e_{Y,i}^2 \le \Big(\frac1n\sum_{i=1}^n \big(\E[r_{X,i}^2 \mid \psi(Z_i)]\big)^2\Big)^{1/2} \Big(\frac1n\sum_{i=1}^n e_{Y,i}^4\Big)^{1/2}.
\]
The first factor is bounded in probability by Markov's inequality, since $\E\big[(\E[r_{X,i}^2 \mid \psi(Z_i)])^2\big] \le \E[r_{X,i}^4] < \infty$ by Jensen's inequality, and the second is $o_P(1)$ by \eqref{eq:fourth}. For $\epsilon > 0$, Chebyshev's inequality conditionally on $\mathcal G_k$ gives $P(|S_k| > \epsilon) \le \E\big[\min\{1, \epsilon^{-2}\E[S_k^2 \mid \mathcal G_k]\}\big] \to 0$ by bounded convergence, and summing over the $K$ folds gives \eqref{eq:cross}. This is the argument of \citet{shah2020} for the GCM, with conditioning on the other folds in place of $H_0$.

Next, we linearize $\hat\rho^\pm_\psi$. Expanding $\hat r_{X,i}\hat r_{Y,i} = r_{X,i}r_{Y,i} - r_{X,i}e_{Y,i} - r_{Y,i}e_{X,i} + e_{X,i}e_{Y,i}$ gives
\[
\hat\kappa_\psi = \frac1n\sum_{i=1}^n r_{X,i} r_{Y,i} + o_P(n^{-1/2}),
\]
because $\frac1n\sum_i r_{X,i}e_{Y,i}$ and $\frac1n\sum_i r_{Y,i}e_{X,i}$ are $o_P(n^{-1/2})$ by \eqref{eq:cross} and, as for the GCM \citep{shah2020}, $\big|\frac1n\sum_i e_{X,i}e_{Y,i}\big| \le (M_X M_Y)^{1/2} = o_P(n^{-1/2})$ by the Cauchy--Schwarz inequality. Similarly, expanding $\hat r_{X,i}^2 = r_{X,i}^2 - 2r_{X,i}e_{X,i} + e_{X,i}^2$ gives
\[
\hat\sigma_\psi^2(X) = \frac1n\sum_{i=1}^n r_{X,i}^2 - \frac2n\sum_{i=1}^n r_{X,i}e_{X,i} + M_X = \frac1n\sum_{i=1}^n r_{X,i}^2 + o_P(n^{-1/2}),
\]
since the middle term is $o_P(n^{-1/2})$ by \eqref{eq:cross} and $M_X = o_P(n^{-1/2})$, and likewise for $Y$. The vectors $(r_{X,i}r_{Y,i}, r_{X,i}^2, r_{Y,i}^2)$ are independent and identically distributed with mean $\theta := (\kappa_\psi, \sigma_\psi^2(X), \sigma_\psi^2(Y))$ and, by the fourth moments in \cref{ass:product}, a finite covariance matrix. By the central limit theorem, $\sqrt n\,\big((\hat\kappa_\psi, \hat\sigma_\psi^2(X), \hat\sigma_\psi^2(Y)) - \theta\big)$ therefore converges in distribution to a normal distribution with mean zero and the covariance matrix of these vectors; in particular, the three estimates are consistent, and $\hat\sigma_\psi(X)$ and $\hat\sigma_\psi(Y)$ are positive with probability tending to one. The map $g(a, b, c) := a/\sqrt{bc}$ is continuously differentiable at $\theta$, since the variances are positive, with gradient $\nabla g(\theta) = \big(1/(\sigma_\psi(X)\sigma_\psi(Y)), -\rho^\pm_\psi/(2\sigma_\psi^2(X)), -\rho^\pm_\psi/(2\sigma_\psi^2(Y))\big)$, $\hat\rho^\pm_\psi = g(\hat\kappa_\psi, \hat\sigma_\psi^2(X), \hat\sigma_\psi^2(Y))$ and $\rho^\pm_\psi = g(\theta)$. The delta method gives
\[
\sqrt n\,(\hat\rho^\pm_\psi - \rho^\pm_\psi) \to N\Big(0, \Var\big(\nabla g(\theta)^\top (r_{X,i}r_{Y,i}, r_{X,i}^2, r_{Y,i}^2)\big)\Big) = N(0, v_\psi^2)
\]
in distribution, since $\nabla g(\theta)^\top (r_{X,i}r_{Y,i}, r_{X,i}^2, r_{Y,i}^2) = \xi_{X,i}\xi_{Y,i} - \frac{\rho^\pm_\psi}{2}(\xi_{X,i}^2 + \xi_{Y,i}^2)$.

It remains to show that $\hat v_\psi^2 \to v_\psi^2$ in probability. Let
\[
\eta_i := \xi_{X,i}\xi_{Y,i} - \frac{\rho^\pm_\psi}{2}\big(\xi_{X,i}^2 + \xi_{Y,i}^2\big), \qquad \hat\eta_i := \hat\xi_{X,i}\hat\xi_{Y,i} - \frac{\hat\rho^\pm_\psi}{2}\big(\hat\xi_{X,i}^2 + \hat\xi_{Y,i}^2\big),
\]
so that $\hat v_\psi^2 = \frac1n\sum_i \hat\eta_i^2$ by \eqref{eq:vhat}. Since $\E[\xi_{X,i}\xi_{Y,i}] = \rho^\pm_\psi$ and $\E[\xi_{X,i}^2] = \E[\xi_{Y,i}^2] = 1$, $\eta_i$ has mean zero, and its variance is $v_\psi^2$ by definition, so $\frac1n\sum_i \eta_i^2 \to v_\psi^2$ in probability by the law of large numbers. By the reverse triangle inequality, $\big|\hat v_\psi - (\frac1n\sum_i \eta_i^2)^{1/2}\big| \le (\frac1n\sum_i (\hat\eta_i - \eta_i)^2)^{1/2}$, so it suffices to show that $\frac1n\sum_i (\hat\eta_i - \eta_i)^2 = o_P(1)$. Let $\delta_{X,i} := \hat\xi_{X,i} - \xi_{X,i} = r_{X,i}\big(1/\hat\sigma_\psi(X) - 1/\sigma_\psi(X)\big) - e_{X,i}/\hat\sigma_\psi(X)$. Since $\hat\sigma_\psi(X) \to \sigma_\psi(X) > 0$ in probability and $\frac1n\sum_i r_{X,i}^4$ is bounded in probability, \eqref{eq:fourth} gives
\[
\frac1n\sum_{i=1}^n \delta_{X,i}^4 \le 8\Big(\frac{1}{\hat\sigma_\psi(X)} - \frac{1}{\sigma_\psi(X)}\Big)^4 \frac1n\sum_{i=1}^n r_{X,i}^4 + \frac{8}{\hat\sigma_\psi^4(X)}\,\frac1n\sum_{i=1}^n e_{X,i}^4 = o_P(1),
\]
and likewise for $\delta_{Y,i} := \hat\xi_{Y,i} - \xi_{Y,i}$. Now
\begin{multline*}
\hat\eta_i - \eta_i = \delta_{X,i}\xi_{Y,i} + \xi_{X,i}\delta_{Y,i} + \delta_{X,i}\delta_{Y,i} - \frac{\hat\rho^\pm_\psi}{2}\big(2\xi_{X,i}\delta_{X,i} + \delta_{X,i}^2 + 2\xi_{Y,i}\delta_{Y,i} + \delta_{Y,i}^2\big) \\ - \frac{\hat\rho^\pm_\psi - \rho^\pm_\psi}{2}\big(\xi_{X,i}^2 + \xi_{Y,i}^2\big).
\end{multline*}
Since $(\sum_{j=1}^m a_j)^2 \le m\sum_{j=1}^m a_j^2$, it suffices that the average of the square of each term is $o_P(1)$. For the terms with a $\delta$, this follows from the Cauchy--Schwarz inequality, for instance $\frac1n\sum_i \delta_{X,i}^2\xi_{Y,i}^2 \le (\frac1n\sum_i \delta_{X,i}^4)^{1/2}(\frac1n\sum_i \xi_{Y,i}^4)^{1/2}$, since $\frac1n\sum_i \xi_{X,i}^4$ and $\frac1n\sum_i \xi_{Y,i}^4$ are bounded in probability and $\hat\rho^\pm_\psi$ is consistent. For the last term, it follows from $\hat\rho^\pm_\psi - \rho^\pm_\psi = o_P(1)$. Hence $\hat v_\psi^2 \to v_\psi^2$ in probability.

Finally, since $v_\psi^2 > 0$, Slutsky's lemma gives $\sqrt n\,(\hat\rho^\pm_\psi - \rho^\pm_\psi)/\hat v_\psi \to N(0, 1)$ in distribution. 
\end{proof}

\subsection{Proof of \texorpdfstring{\cref{prop:robust}}{Proposition \ref{prop:robust}}}\label{app:robust}

The construction of $p_B$ in \cref{sec:robust} rests on the following fact. If $W \sim N(\mu, 1)$, then for every $r \ge 0$,
\[
P(|W| > r) = 2 - \Phi(r - \mu) - \Phi(r + \mu),
\]
which is even in $\mu$ and nondecreasing in $|\mu|$: for $\mu \ge 0$, its derivative with respect to $\mu$ is $\varphi(r - \mu) - \varphi(r + \mu) \ge 0$, where $\varphi$ is the standard normal density, because $|r - \mu| \le r + \mu$, with strict inequality if $r, \mu > 0$. With $r = \sqrt n\,\hat\rho_\psi/\hat v_\psi$ and $\mu = \sqrt n\,B/\hat v_\psi$, it equals $p_B$ in \eqref{eq:calib}.

\begin{proof}
By \cref{prop:normal}, $T_n := \sqrt n\,(\hat\rho^\pm_\psi - \rho^\pm_\psi)/\hat v_\psi$ converges in distribution to a standard normal random variable, and $\hat v_\psi \to v_\psi > 0$ in probability; hence $\hat\rho^\pm_\psi \to \rho^\pm_\psi$ and $\hat\rho_\psi \to \rho_\psi$ in probability. Under $H_0$ and a valid tolerance, $\rho_\psi \le B$ by \cref{cor:sup}(ii). We consider three cases.

If $\rho_\psi < B$, then $(\hat\rho_\psi - B)/\hat v_\psi \to (\rho_\psi - B)/v_\psi < 0$ in probability, so $\sqrt n\,(\hat\rho_\psi - B)/\hat v_\psi \to -\infty$ in probability. Since $p_B \ge 1 - \Phi\big(\sqrt n\,(\hat\rho_\psi - B)/\hat v_\psi\big)$, $p_B \to 1$ in probability, and $P(p_B < \alpha) \to 0$.

If $\rho_\psi = B = 0$, then $\sqrt n\,\hat\rho_\psi/\hat v_\psi = |T_n|$, so $p_B = p_0 = 2\big(1 - \Phi(|T_n|)\big)$ is the two-sided normal $p$-value of $T_n$. By the continuous mapping theorem, $p_0$ converges in distribution to the uniform distribution on $(0, 1)$, and $P(p_B < \alpha) \to \alpha$.

If $\rho_\psi = B > 0$, we may take $\rho^\pm_\psi = B$, since replacing $X$ by $-X$ changes the signs of $\rho^\pm_\psi$ and $\hat\rho^\pm_\psi$ but leaves $\hat\rho_\psi$, $\hat v_\psi$, and $p_B$ unchanged. On the event $\hat\rho^\pm_\psi > 0$, whose probability tends to one, $\hat\rho_\psi = \hat\rho^\pm_\psi$, and \eqref{eq:calib} reads
\[
p_B = 1 - \Phi(T_n) + 1 - \Phi\Big(\frac{\sqrt n\,(\hat\rho^\pm_\psi + B)}{\hat v_\psi}\Big).
\]
The second term tends to zero in probability, because $(\hat\rho^\pm_\psi +
B)/\hat v_\psi \to 2B/v_\psi > 0$ in probability. The first term is the
one-sided normal $p$-value of $T_n$, so, by the continuous mapping theorem and
Slutsky's lemma, $p_B$ converges in distribution to the uniform distribution on
$(0, 1)$, and $P(p_B < \alpha) \to \alpha$.

In each case, $P(p_B < \alpha)$ converges to a limit that is at most $\alpha$, with equality if and only if $\rho_\psi = B$.
\end{proof}

\subsection{Consistency of the robust test}\label{app:power}

\begin{corollary}[Consistency of the robust test]\label[corollary]{cor:power}
Suppose that \cref{ass:product} holds and that $v_\psi^2 > 0$. If
\[
\sqrt{(1 - R^2_\psi(X))(1 - R^2_\psi(Y))}\,\rho_{\id} > B + \sqrt{R^2_\psi(X) R^2_\psi(Y)},
\]
then $P(p_B < \alpha) \to 1$ for every $\alpha \in (0, 1)$.
\end{corollary}
\begin{proof}
By \cref{thm:pair} and $|\lambda| \le 1$,
\[
\rho_\psi \ge \sqrt{(1 - R^2_\psi(X))(1 - R^2_\psi(Y))}\,\rho_{\id} - \sqrt{R^2_\psi(X) R^2_\psi(Y)} > B.
\]
By \cref{prop:normal}, $\hat\rho_\psi \to \rho_\psi$ and $\hat v_\psi \to v_\psi > 0$ in probability, so $(\hat\rho_\psi - B)/\hat v_\psi \to (\rho_\psi - B)/v_\psi > 0$ in probability and $\sqrt n\,(\hat\rho_\psi - B)/\hat v_\psi \to \infty$ in probability. Since $p_B \le 2\big(1 - \Phi(\sqrt n\,(\hat\rho_\psi - B)/\hat v_\psi)\big)$, it follows that $P(p_B < \alpha) \to 1$.
\end{proof}

\subsection{Proof of \texorpdfstring{\cref{prop:bench}}{Proposition \ref{prop:bench}}}

\begin{proof}

Write $\sigma^2_{(\psi,\psi')}(X)$ for the residual variance of $X$ given the
embedding $(\psi, \psi')$, so that $R^2_{\psi \to \psi'}(X) = 1 -
\sigma^2_{(\psi,\psi')}(X)/\sigma^2_\psi(X)$. Since $\E[X \mid \psi(Z),
\psi'(Z)]$ minimizes the mean squared error among square-integrable functions of
$(\psi(Z), \psi'(Z))$, which include $\E[X \mid \psi(Z)]$, we have
$\sigma^2_{(\psi,\psi')}(X) \le \sigma^2_\psi(X)$, and likewise
$\sigma^2_{\id}(X) \le \sigma^2_{(\psi,\psi')}(X)$. Dividing by
$\sigma^2_\psi(X)$ gives $0 \le R^2_{\psi \to \psi'}(X) \le 1 -
\sigma^2_{\id}(X)/\sigma^2_\psi(X) = R^2_\psi(X)$, and likewise for $Y$;
multiplying the two bounds gives $\rbench \le \sqrt{R^2_\psi(X) R^2_\psi(Y)}$.
If $\E[X \mid Z]$ is a function of $(\psi(Z), \psi'(Z))$, then $\E[X \mid
\psi(Z), \psi'(Z)] = \E\big[\E[X \mid Z] \,\big|\, \psi(Z), \psi'(Z)\big] = \E[X
\mid Z]$, so $\sigma^2_{(\psi,\psi')}(X) = \sigma^2_{\id}(X)$ and $R^2_{\psi \to
\psi'}(X) = R^2_\psi(X)$. If the same holds for $Y$, the bound is an equality.
\end{proof}

\section{Sensitivity parameters and extensions}\label{app:kernel}

\subsection{Relation to omitted variable bias}\label{app:ovb}

\citet{cinelli2020} write the omitted variable bias of a regression coefficient
in a linear model as the geometric mean of two partial $R^2$ values times a
scale factor, and build robustness values and benchmarks on it. The bound of
\cref{thm:pair} has the same product form, but it concerns the residual
correlation $\rho_\psi$, so it is scale-free, and it assumes no linear model:
$R^2_\psi(X)$ and $R^2_\psi(Y)$ compare residual variances around conditional
expectations, and \eqref{eq:pair} holds for every distribution.

\citet{chernozhukov2022lss} bound the omitted variable bias of linear
functionals of a regression function. Their Theorem~1 writes the bias as
$\E[(g_s - g)(\alpha_s - \alpha)]$, the covariance of the errors in the
regression and in its Riesz representer when the omitted variables are dropped,
and their Corollary~1 bounds it through nonparametric partial $R^2$ values. The
expected conditional covariance $\E[\Cov(X, Y \mid \psi(Z))] = \E[(X - \E[X \mid
\psi(Z)])\, Y]$ is a linear functional of the regression $\E[Y \mid X, \psi(Z)]$
with representer $X - \E[X \mid \psi(Z)]$. The long versions condition on $Z$,
so the representer error is $\Delta_\psi(X)$, and under $H_0$ the long
functional is zero and the regression error is $\E[Y \mid X, \psi(Z)] - \E[Y
\mid Z]$. Their identity and \eqref{eq:pair} thus describe the same bias, but
\eqref{eq:pair} splits it symmetrically into $\Delta_\psi(X)$ and
$\Delta_\psi(Y)$, and the $Y$-side term does not involve $X$. This is what makes
$R^2_\psi(X)$ and $R^2_\psi(Y)$ the natural sensitivity parameters; each is a
correlation ratio, the kind of nonparametric partial $R^2$ used in their
Corollary~1.

\subsection{Robustness values}\label{app:rv}

\begin{definition}[Robustness value]\label[definition]{def:rv}
For a level $\alpha$, $\RV_\alpha$ is the smallest tolerance $B \ge 0$ at which the robust test does not reject. That is, $\RV_\alpha := 0$ if $p_0 \ge \alpha$, and otherwise $\RV_\alpha$ is the unique $B$ with $p_B = \alpha$.
\end{definition}
By the fact at the start of \cref{app:robust}, with $\mu = \sqrt n\,B/\hat v_\psi$, $p_B$ is nondecreasing in $B$, and strictly increasing if $\hat\rho_\psi > 0$. It is also continuous in $B$ and tends to $1$ as $B \to \infty$. If $p_0 < \alpha$, then $\hat\rho_\psi > 0$, so exactly one $B$ has $p_B = \alpha$. In both cases, the robust test with tolerance $B$ rejects if and only if $B < \RV_\alpha$.

\subsection{Identification with replicates}

The sensitivity parameters involve $\E[X \mid Z]$ and $\E[Y \mid Z]$, which cannot be estimated directly from a text-valued $Z$, but they are identified when conditionally independent replicates are available, as in the application, where several replies share a $Z$, and in Experiment~3, where the generating model can be sampled. \resFourTruthM\ regenerated comments and summaries per post identify the three parameters for every embedding (\cref{sec:exp4}).

\begin{proposition}[Identification with
replicates]\label[proposition]{prop:replicates} Let $X'$ be square integrable
and conditionally independent of $(X, Y)$ given $Z$, with the same conditional
law given $Z$ as $X$. Then $\Cov(X, X') = \Var(\E[X \mid Z])$, so that
$\sigma_{\id}^2(X) = \Var(X) - \Cov(X, X')$, $\E[\Delta_\psi(X)^2] = \Cov(X, X')
- \Var(\E[X \mid \psi(Z)])$, and $R^2_\psi(X)$ are identified from the joint law
of $(X, X', \psi(Z))$. If $Y'$ is such a replicate of $Y$, conditionally
independent of $X$ given $Z$, then $\Cov(X, Y') = \Cov(\E[X \mid Z], \E[Y \mid
Z])$, so that $\Cov(\Delta_\psi(X), \Delta_\psi(Y)) = \Cov(X, Y') - \Cov(\E[X
\mid \psi(Z)], \E[Y \mid \psi(Z)])$ and $\lambda$ are identified.
\end{proposition}

\begin{proof}
$\E[X X'] = \E\big[\E[X \mid Z]\, \E[X' \mid Z]\big] = \E[(\E[X \mid Z])^2]$ by conditional independence and equality of the conditional laws; subtracting $\E[X]\E[X'] = (\E[X])^2$ gives the first claim. Since $\E[X \mid Z] = \E[X \mid \psi(Z)] + \Delta_\psi(X)$ with $\E[\Delta_\psi(X) \mid \psi(Z)] = 0$, $\Var(\E[X \mid Z]) = \Var(\E[X \mid \psi(Z)]) + \E[\Delta_\psi(X)^2]$, and $\Var(X) = \sigma_{\id}^2(X) + \Var(\E[X \mid Z])$. For $Y'$, $\E[X Y'] = \E\big[\E[X \mid Z]\, \E[Y \mid Z]\big]$, and the same decomposition gives $\Cov(\E[X \mid Z], \E[Y \mid Z]) = \Cov(\E[X \mid \psi(Z)], \E[Y \mid \psi(Z)]) + \Cov(\Delta_\psi(X), \Delta_\psi(Y))$.
\end{proof}
Second embeddings $\psi'$ that depend on $Z$ and on independent noise $U$, such
as a replicate $X'$ or the mean of several, are covered by \cref{prop:bench}
after replacing $Z$ by $(Z, U)$. Since $U$ is independent of $(X, Y, Z)$, the
hypothesis $H_0$, the conditional means given $Z$, and hence $\Delta_\psi(X)$,
$\Delta_\psi(Y)$, $R^2_\psi(X)$, and $R^2_\psi(Y)$ are unchanged.

\section{Implementation details}\label{app:impl}

\Cref{alg:main} summarizes the procedure of \cref{sec:sens}. The rest of this appendix describes the regressions and the details of each experiment.

\begin{algorithm}[h]
\caption{Sensitivity analysis for an embedded conditional independence test}\label{alg:main}
\begin{algorithmic}[1]
\REQUIRE Observations $(X_i, Y_i)_{i \le n}$, embeddings $\psi(Z_i)$ and $\psi'(Z_i)$, number of folds $K$, level $\alpha$, optional tolerance $B$.
\STATE Cross-fit regressions of $X$ and $Y$ on $\psi(Z)$; form residuals $\hat r_{X,i}$, $\hat r_{Y,i}$.
\STATE Compute $\hat\rho^\pm_\psi$ by \eqref{eq:rhohat} and $\hat\rho_\psi = |\hat\rho^\pm_\psi|$, $\hat v_\psi$ by \eqref{eq:vhat}, the $p$-value $p_0$ of \cref{sec:gcm}, and $\RV_\alpha$ of \cref{def:rv}.
\STATE Cross-fit regressions of $X$ and $Y$ on $(\psi(Z), \psi'(Z))$; compute $\hat R^2_{\psi \to \psi'}(X)$, $\hat R^2_{\psi \to \psi'}(Y)$ and $\rbenchhat$ (\cref{sec:bench}).
\STATE If a tolerance $B$ is given, reject $H_0$ if $p_B < \alpha$, with $p_B$ from \eqref{eq:calib}.
\ENSURE $p_0$, $\RV_\alpha$, $\rbenchhat$, and the robust decision.
\end{algorithmic}
\end{algorithm}

\subsection{Regressions and cross-fitting}\label{app:impl-reg}
All conditional means are estimated with $K = 5$ folds. The regressions on $\psi(Z)$ and on $(\psi(Z), \psi'(Z))$ use the same folds, and inputs are standardized on the training folds. In Experiments~1 and~2, we use least squares when every training fold has at least as many rows as the conditioning set has coordinates, which on the grids of both experiments is the condition $n \ge d$, and ridge regression with the penalty chosen by leave-one-out cross-validation otherwise; the two regressions of a benchmark use the procedure of the richer set. In Experiment~1 every regression is therefore least squares, and in Experiment~2 only the \resTwoRichDim-dimensional set at $n < \resTwoOLSFromRich$ is fitted by ridge regression. In Experiment~3 and in the application, every regression is least squares, and we report an embedding only if every training fold has at least 2.5 times as many rows as the embedding has coordinates.

Each procedure has a drawback. The estimation variance of least squares near $n = d$, of order $d/(n_{\mathrm{train}} - d)$, is not captured by the standard error of \cref{prop:normal}, and the residual correlation test then loses its level even when the discarded parts are uncorrelated. With $\psi = E_1$ and $\lambda = 0$ in Experiment~2, the rejection rate under $H_0$ is \resTwoTypeIOrthSmall\ at $n = \resTwoNmin$, where the training folds hold \resTwoTrainRowsSmall\ rows for \resTwoDimWeak\ coordinates, and \resTwoTypeIOrthOneK\ at $n = 1000$, against \resTwoTypeIOrthFromTwoKMin\ to \resTwoTypeIOrthFromTwoKMax\ from $n = 2000$ on. Ridge regression has no such variance, but a penalty chosen for prediction shrinks the fitted means and can leave correlated parts of the discarded information in both residuals. On the sufficient \resTwoRichDim-dimensional set of Experiment~2 it rejects at rate \resTwoTypeIRichRidgeLarge\ at $n = \resTwoNmax$, where $\hat\rho_\psi = \resTwoRhoRichRidgeLarge$, which is why we use least squares wherever $n \ge d$. The sentiment direction of Experiment~3 and the prognostic score of the application are ridge predictions; they define variables and are not regressions of the test.

\subsection{Standard errors and \texorpdfstring{$p$}{p}-values}\label{app:impl-se}
The standard error of $\hat\rho^\pm_\psi$ is $\hat v_\psi/\sqrt n$, with $\hat v_\psi^2$ from \eqref{eq:vhat}. Where rows are clustered, as in the application, the estimated influence function is summed within clusters. The $p$-values $p_0$ and $p_B$ of \cref{sec:gcm,sec:robust} use this standard error, and $\RV_\alpha$ is computed by bisection.

\subsection{Demonstrated insufficiency}\label{app:impl-bench}
We estimate $R^2_{\psi \to \psi'}(X)$ by one minus the ratio of the cross-fitted mean squared error of the regression on $(\psi(Z), \psi'(Z))$ to that of the regression on $\psi(Z)$, truncated at zero, with the same folds for both regressions. Each cross-fitted mean squared error is the residual variance plus the mean squared estimation error of its regression. To leading order, the estimate is therefore $R^2_{\psi \to \psi'}(X)$ minus the excess estimation error of the richer regression over the base regression, relative to $\sigma^2_\psi(X)$, and the truncated estimate is zero unless $R^2_{\psi \to \psi'}(X)$ exceeds this excess, which is of order $d'/n$ for least squares with $d'$ added coordinates. A one-dimensional prognostic score adds almost no estimation error, which is why it demonstrates insufficiency more reliably than a richer embedding (\cref{sec:bench}). In Experiment~2, the \resTwoDimStrong\ coordinates of $E_2$ demonstrate none of an insufficiency that they explain in full (\cref{app:figs}).

\subsection{Corpus and encoders}\label{app:impl-corpus}
The 20~Newsgroups posts have headers, footers and quoted text removed. We drop posts with fewer than 30 words and cut longer posts at 300 words, leaving \resTwoN\ posts. $E_1$ is all-MiniLM-L6-v2 ($d = \resTwoDimWeak$) and $E_2$ is all-mpnet-base-v2 ($d = \resTwoDimStrong$), both used through the sentence-transformers library \citep{reimers2019} with default pooling and a window of 256 word pieces.

\subsection{Experiment~1}\label{app:impl-exp1}
Let $\theta_{X, q:p}$ and $\theta_{Y, q:p}$ be the coefficients of the dropped coordinates $Z_{q:p}$. Then $\Delta_{\psi_q}(X) = \theta_{X, q:p}^\top Z_{q:p}$, $R^2_{\psi_q}(X) = \|\theta_{X, q:p}\|^2 / (\|\theta_{X, q:p}\|^2 + 1)$, and $\lambda$ is the cosine of the angle between $\theta_{X, q:p}$ and $\theta_{Y, q:p}$. We set $R^2_{\psi_q}(X) = R^2_{\psi_q}(Y) = \resOneRtwoA$, so that $\sqrt{R^2_{\psi_q}(X) R^2_{\psi_q}(Y)} = \resOneBound$, and repeat every configuration \resOneNrep\ times. The power experiment uses $n = \resOnePowerN$, the sufficient embedding and the alternative $Y = \theta_Y^\top Z + \gamma X + \varepsilon_Y$.

\subsection{Experiment~2}\label{app:impl-exp2}
The outcomes are
\begin{align*}
    X &= h_X + c\,g_1 + \varepsilon_X, & Y &= h_Y + c\,(\lambda g_1 + \sqrt{1-\lambda^2}\, g_2) + \varepsilon_Y,
\end{align*}
with independent standard normal $\varepsilon_X, \varepsilon_Y$. Here $h_X$ and $h_Y$ are the standardized first two principal components of $E_1(Z)$, which $E_1$ captures fully. The features $g_1$ and $g_2$ are standardized projections on two random directions of the residual of the standardized $E_2(Z)$ after a least squares regression on $E_1(Z)$ over the corpus, with $g_2$ orthogonalized against $g_1$. Both are linear functions of $[E_1(Z), E_2(Z)]$ and uncorrelated with $E_1(Z)$, so $E_1$ captures none of them and $E_2$ demonstrates all of the insufficiency of $E_1$. The constant $c = \sqrt{r^2/(1-r^2)}$ with $r^2 = 0.1$ sets $R^2_{E_1}(X) \approx R^2_{E_1}(Y) \approx \resTwoPopRtwo$ and $\rho_{E_1} \approx \resTwoPopBound$ at $\lambda = 1$. Under the alternative, $Y$ receives an additional term $\gamma X$. The population values of $R^2_{E_1}(X)$, $R^2_{E_1}(Y)$, $\rho_{E_1}$ and the benchmark are computed by cross-fitted least squares on \resTwoPopN\ posts; the partial $R^2$ is \resTwoPopRtwo\ rather than $0.1$ because the cross-fitted regression on the \resTwoDimWeak\ coordinates of $E_1(Z)$ leaves an estimation error of order $d/n$ in the residual. The oracle tolerance of the robust test in \cref{fig:exp123} is this cross-fitted value, so it includes the estimation error of a least squares regression at $n = \resTwoPopN$; the design value under the exact linear projection is $0.1$. The population benchmark of $E_2$ is \resTwoPopBench, the full insufficiency, because $\E[X \mid Z]$ and $\E[Y \mid Z]$ are linear in $[E_1(Z), E_2(Z)]$; its estimate uses the \resTwoRichDim-dimensional set and follows that set's regression rule. The rejection rates under $H_0$ use \resTwoNrep\ draws of the outcomes per configuration, the sensitivity analysis \resTwoNrepBench\ draws, and the power curves, at $n = \resTwoPowerN$, \resTwoNrepPower\ draws. Kernel ridge regression with a Gaussian kernel in place of least squares gives the same picture for $\psi = E_1$ (population $R^2_{E_1}(X) = \resTwoPopRtwoKRR$, benchmark \resTwoPopBenchKRR) and inflates the level of the \resTwoRichDim-dimensional set, as ridge regression does.

\begin{remark}[The null hypothesis for generated text]\label[remark]{rem:llm}
Suppose the comment $C_i = g_W(Z_i, U_i)$ is generated by a language model with weights $W$ from the post $Z_i$ and fresh sampling noise $U_i$, and that $X_i$ is a function of $C_i$. Then $P(X_i \in \cdot \mid Z_i, Y_i, W) = P(X_i \in \cdot \mid Z_i, W)$ for every $W$, however the weights were obtained, so conditional on the model, $X \indep Y \mid Z$ holds by construction. A dependence can only arise if $W$ is treated as random and dependent on the sample, for instance through memorization \citep{simnacher2026ecit}; this is a single shared realization for which the i.i.d.\ calibration of a test is not obviously appropriate. Text generated from real $Z$ by a fixed model is therefore a benchmark in which the null hypothesis is exact, and a rejection on such data is a type~I error, caused by insufficiency of $\psi$ or by miscalibration.
\end{remark}

\subsection{Experiment~3}\label{app:impl-exp3}
The generator is \resFourModel\ \citep{qwen2025} (bfloat16, Hugging Face transformers on a laptop GPU) with the system prompt and the two instructions of \cref{app:prompts}, the persona drawn per post with a fixed seed, sampling at temperature $1$ with nucleus probability $0.95$ and at most 160 new tokens, and the model's thinking mode disabled. The comment ends with a line in which it labels its own sentiment as positive, neutral or negative (coded $+1, 0, -1$), which we parse and remove before the comment is used; all but one of the \resThreeN\ comments carried a label, and that one is dropped. The label is the target of a cross-fitted ridge regression on the $E_1$ embedding of the comment, and its out-of-fold predictions on the $E_1$ embeddings of the comment and of the summary are $X$ and $Y$. The direction is cross-fitted with the same five folds as the residual regressions, so the comments of a fold enter, through the direction, the values of $X$ and $Y$ in the other folds; for a ridge direction at $n = \resThreeN$ this leakage is small. Both encoders read the first 256 word pieces of a post, and only \resThreeTruncFrac\ of the posts are longer.

The hidden states come from a forward pass over the summary prompt, which contains the post and a fixed instruction and nothing else, in batches of ten posts in a fixed order, since left padding changes the states slightly. The four embeddings are the mean over the prompt tokens and the state at the last prompt token, each at the output of the last layer and of layer \resFourMidLayer\ of \resFourLayers; principal components are fitted on the \resThreeN\ posts. By the rule above, every embedding of at most 960 coordinates is reported, and the \resTwoRichDim- and \resFourHidden-dimensional ones are not.

We drew the replicates with vLLM, with the same bfloat16 weights, prompts, sampling settings and label parsing, and the persona redrawn per draw: \resFourTruthM\ comments and \resFourTruthM\ summaries for each of the \resThreeN\ posts, 120{,}000 texts in all. \resFourTruthN\ posts have all draws complete and a labelled observed comment. The regenerated comments match the observed ones in length, label frequencies and the share written in the first person, so the two runs sample the same law. The analysis follows \cref{prop:replicates}. $\Var(\E[X \mid Z])$ and $\Cov(\E[X \mid Z], \E[Y \mid Z])$ are the between-post components of a balanced one-way analysis of variance over the draws. $\Var(\E[X \mid \psi(Z)])$ is estimated by the sample covariance of the cross-fitted prediction $\hat m_X(\psi(Z_i))$ with $X_i$, which, unlike the cross-fitted mean squared error, is not inflated by the estimation error of the regression, and $\Cov(\E[X \mid \psi(Z)], \E[Y \mid \psi(Z)])$ by the average of the two cross covariances of a prediction with the other variable. Then $\sigma^2_\psi(X) = \Var(X) - \Var(\E[X \mid \psi(Z)])$ gives $R^2_\psi(X)$, $R^2_\psi(Y)$, $\lambda$ and the true insufficiency, with 95\% percentile intervals from 200 bootstrap samples of posts. Each null dataset takes one comment draw and one summary draw per post, independently, is scored with the sentiment direction of the observed data, and is tested exactly like the observed data. \Cref{fig:exp4identity} in \cref{app:figs} compares $\hat\rho_\psi$ on these datasets with the prediction and the bound of \cref{thm:pair} for every embedding.

The alternatives re-pair the draws within each post. With probability $\pi$, a comment draw is paired with the summary draw whose rank among the \resFourPowerP\ summary draws of the post equals its own rank among the comment draws, and otherwise with an independent draw. Only the pairing changes, so every quantity identified above is unchanged, and the conditional covariance given the post is $\pi$ times that of the sorted draws. Hence $\pi$ sets $\rho_{\id}$ exactly, up to \resFourPowerRange\ at $\pi = 1$, and the induced dependence is positive, the direction of the bias here. By \cref{thm:pair}, the residual correlation with $E_1$ is $\resFourPowerFactorEone\,\rho_{\id} + \resFourTruthPredRhoEone$, and the mean of $\hat\rho_\psi$ over the \resFourPowerR\ datasets per value follows this line to within \resFourPowerRhoDiffMaxEone. The reference test residualizes $X$ and $Y$ on the means of the draws of the post that the dataset does not use. Under $H_0$ these means are independent of the dataset given the post, and their error enters the residual variances, so the reference approximates the test with the post observed. \Cref{fig:exp4power} in \cref{app:figs} shows the rejection rates for $E_1$, $E_2$ and the generator's state.

\subsection{Application}\label{app:impl-app}
From the \resAppNops\ conversations of the ChangeMyView corpus we keep every reply whose parent is a comment rather than the post, written by neither the original poster, a bot, nor the author of the parent, not deleted, of at least twenty words and with a score; the parent must not be deleted and must have at least five words. We keep the earliest reply per user and parent, at most \resDeepMaxPerParent\ replies per parent (sampled at random with a fixed seed), and parents with at least two such replies. This gives \resDeepN\ rows under \resDeepNparents\ parent comments in \resDeepNposts\ posts, \resDeepRowsPerParent\ rows per parent on average. $Z$ is the parent comment, of median length \resDeepParentWords\ words; \resDeepParentDepthOne\ of the parents are top-level replies to the post and the rest lie deeper in the thread. Scores are Reddit scores at collection time, and $Y = \mathrm{sign}(s)\log(1+|s|)$ for a score $s$. The delta indicator marks the comment that a delta was awarded to, the ancestor of the DeltaBot confirmation written by the awarded user (\resDeepDeltaRate\ of the rows). Besides the politeness rating and the log length, we use as $X$ a lexicon score, politeness markers minus hostile markers per hundred words.

Politeness is rated by Jev \citep{typesafe2026jev}, a decision model that answers a rubric question with a probability-weighted score instead of generating text. The reply is passed as the state, the instruction is ``Rate the politeness of this online comment. Judge tone and courtesy only, not the opinion expressed.'', and the five levels of the rubric are rude or hostile, somewhat impolite or curt, neutral, polite, and very polite and courteous, coded 1 to 5, so that the score is continuous (median \resDeepRaterMedian, standard deviation \resDeepRaterSd). All \resDeepRated\ replies were rated in twelve minutes at a cost of \resDeepRatingCost\ USD. As a check, \resDeepAgreeN\ of the replies were also rated by \resThreeModel\ \citep{anthropic2025haiku} at temperature $0$, with the system prompt ``You rate the politeness of online comments for a research dataset. Judge tone and courtesy only, not the opinion expressed. Answer with a single digit.'' and the user prompt ``Comment: \texttt{<{}<{}<}[reply]\texttt{>{}>{}>} Rate the politeness of this comment from 1 (rude or hostile) to 5 (very polite and courteous). Answer with a single digit.'' The two ratings have correlation \resDeepAgreeCorr\ (rank correlation \resDeepAgreeSpearman). The rounded score agrees exactly in \resDeepAgreeExact\ of the cases and within one point in \resDeepAgreeWithinOne, and its mean rises from \resDeepRaterAtHaikuOne\ to \resDeepRaterAtHaikuFive\ across the five levels of the other rating, so the two scales differ by a shift of about half a point and not in order.

The embeddings of $Z$ are $E_1$ and $E_2$ of the parent comment, which fits into the 256 word-piece window for \resDeepParentFit\ of the parents; its surface statistics $S$ (length, sentence count, character classes, mean word length and presence of a URL) and 300 TF-IDF singular directions $L$; its metadata $M$ (log length, signed log score, whether the original poster wrote it, depth, log number of replies, whether it quotes its parent, presence of a URL); the $E_1$ embedding $P$ of the post it belongs to; combinations of these with $E_1$; and the prognostic score, the cross-fitted ridge predictions of $X$ and of $Y$ from $[L, S, M]$. Folds are grouped by post, and standard errors sum the influence function within posts. The baseline demeans and permutes $X$ within parents. The variance components are unbiased analysis-of-variance estimates for unbalanced groups with parents as groups, and $\hat\lambda$ is the between-parent covariance of the two variables divided by the product of their between-parent standard deviations.

\subsection{Computation}\label{app:impl-comp}
Experiments~1 and~2 run on a laptop, in minutes and in about one hour, respectively. The generation of Experiment~3 took about \resFourGenHours\ hours on the laptop's GPU and the hidden states ten minutes; the replicates took five minutes on one H100 GPU.

\subsection{Licenses}\label{app:impl-lic}
\resFourModel\ and the two sentence encoders are released under the Apache~2.0 license. The 20~Newsgroups corpus and the ChangeMyView corpus of ConvoKit are distributed by their creators for research use without a stated license. The ratings by Jev and \resThreeModel\ were obtained through the providers' commercial APIs under their terms of use.

\section{Additional figures and results}\label{app:figs}

\textbf{Experiment~1.} \Cref{fig:exp1a} plots $\hat\rho_\psi$ under $H_0$ against $\sqrt{R^2_\psi(X) R^2_\psi(Y)}$ for five values of $\lambda$. The points lie on the lines $\lambda\sqrt{R^2_\psi(X) R^2_\psi(Y)}$ of \eqref{eq:pair}, and the bound is attained at $\lambda = 1$. \Cref{fig:exp1b} (left) shows the power at $n = \resOnePowerN$ with the sufficient embedding. The power curve of the robust test is that of the residual correlation test shifted to the right by roughly $B$, as \cref{cor:power} predicts: at $\gamma = 0.1$, the residual correlation test has power \resOnePowerNaive\ and the robust test with $B = \resOnePowerBudgetRef$ has power \resOnePowerRobust. \Cref{fig:exp1b} (right) shows $\hat\rho_\psi$ and $\RV_\alpha$ under $H_0$ with the insufficient embedding at $\lambda = 1$. The statistic $\hat\rho_\psi$ is centred at $\rho_{\psi_q} = \resOneRVTrue$, and $\RV_\alpha$ lies below this population value in \resOneRVCoverage\ of the runs.

\begin{figure}[h]
\centering
\figorbox{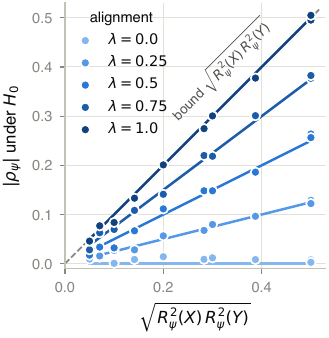}{0.4\linewidth}
\caption{Experiment~1, $\hat\rho_\psi$ under $H_0$ against $\sqrt{R^2_\psi(X) R^2_\psi(Y)}$ for five values of $\lambda$. Solid lines are $\lambda\sqrt{R^2_\psi(X) R^2_\psi(Y)}$, the value under $H_0$ by \eqref{eq:pair}; the dashed line is the bound.}
\label{fig:exp1a}
\end{figure}

\begin{figure}[h]
\centering
\figorboxtall{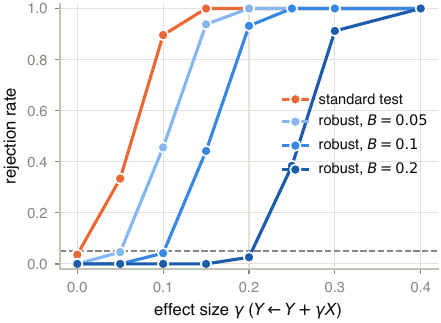}{0.48\linewidth}\hfill
\figorboxtall{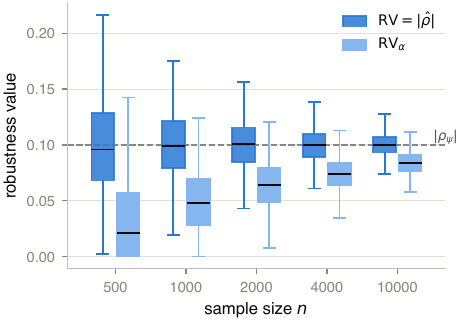}{0.48\linewidth}
\caption{Experiment~1. Left: power at $n = \resOnePowerN$ against the effect size $\gamma$ of the alternative $Y = \theta_Y^\top Z + \gamma X + \varepsilon_Y$, with the sufficient embedding, for the residual correlation test and for the robust test with tolerances $B \in \resOnePowerBudgets$. Right: $\hat\rho_\psi$ (labelled $\mathrm{RV} = |\hat\rho|$ in the legend, the robustness value of the point estimate) and $\RV_\alpha$ against $n$ under $H_0$ with the insufficient embedding at $\lambda = 1$; the dashed line is the population value $\rho_{\psi_q} = \resOneRVTrue$.}
\label{fig:exp1b}
\end{figure}

\textbf{Experiment~2.} \Cref{fig:exp2} shows the sensitivity analysis with $\psi = E_1$ and $\lambda = 1$. The median of $\hat\rho_\psi$ is \resTwoRVmedSmall\ at $n = \resTwoNmin$ and \resTwoRVmed\ at $n = \resTwoNmax$, against the oracle bound \resTwoPopBound, and the median $\RV_\alpha$ rises from \resTwoRValphaSmall\ to \resTwoRValphaLarge. The insufficiency of $E_1$ demonstrated by $E_2$ has population value \resTwoPopBench, all of the truth, but its estimate has median \resTwoBenchSampleMed\ at every $n$: the excess estimation variance of the \resTwoDimStrong\ added coordinates dominates the increment (\cref{app:impl-bench}). \Cref{fig:exp2power} shows the power at $n = \resTwoPowerN$. At $\lambda = 1$, the robust test with the oracle tolerance approximately holds its level (\resTwoRobustBoundNull\ against \resTwoTypeIWeakPowerN\ for the residual correlation test) and loses power at small effects.

\begin{figure}[h]
\centering
\figorboxtall{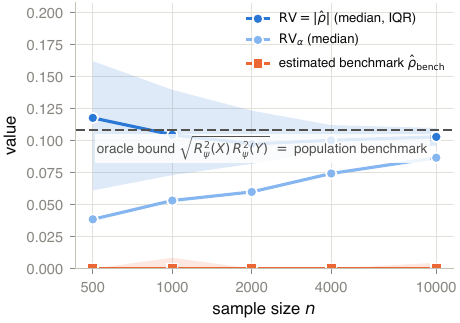}{0.55\linewidth}
\caption{Experiment~2, $\hat\rho_\psi$ (median and interquartile range; labelled $\mathrm{RV} = |\hat\rho|$ in the legend), median $\RV_\alpha$, and the estimated benchmark $\hat\rho_{E_1 \to E_2}$ against $n$ with $\psi = E_1$ and $\lambda = 1$. The dashed line is the oracle bound, which equals the population benchmark.}
\label{fig:exp2}
\end{figure}

\begin{figure}[h]
\centering
\figorboxh{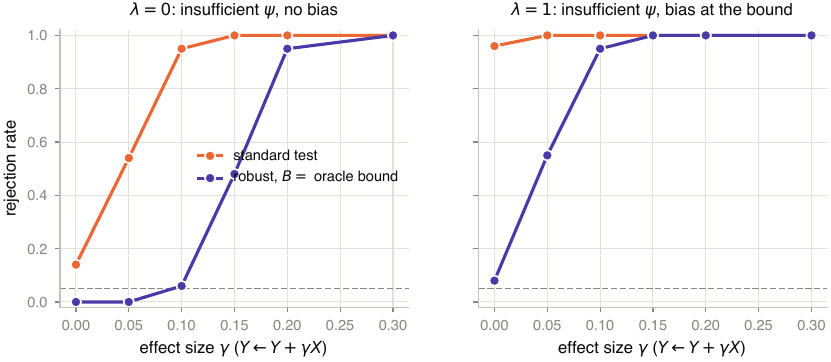}{\linewidth}{3.4cm}
\caption{Experiment~2, power at $n = \resTwoPowerN$ against the effect size $\gamma$ with $\psi = E_1$, for the residual correlation test and for the robust test with tolerance equal to the oracle bound. Left: $\lambda = 0$, where $E_1$ is insufficient but the residual correlation test is unbiased. Right: $\lambda = 1$, where the residual correlation test rejects at rate \resTwoTypeIWeakPowerN\ under $H_0$.}
\label{fig:exp2power}
\end{figure}

\textbf{Experiment~3.} \Cref{fig:exp4views} shows $\hat\rho_\psi$ under $H_0$ for the four embeddings of the generator's state against their dimension, together with the true insufficiency of each, and the rejection rates of the tests for the mean and the last-token state at the last layer. \Cref{fig:exp4identity} checks the identity of \cref{thm:pair}. For every embedding, the mean of $\hat\rho_\psi$ over the \resFourTruthR\ null datasets lies within \resFourTruthPredDiffMax\ of the prediction $|\lambda|\sqrt{R^2_\psi(X) R^2_\psi(Y)}$ identified from the replicates, and below the bound $\sqrt{R^2_\psi(X) R^2_\psi(Y)}$, since $\lambda$ lies between \resFourTruthLambdaMinStates\ and \resFourTruthLambdaMaxStates. \Cref{fig:exp4power} shows the rejection rates against the coupled alternatives for $E_1$, $E_2$ and 384 principal components of the mean state at the last layer; the three embeddings give the same picture.

\begin{figure}[h]
\centering
\figorboxh{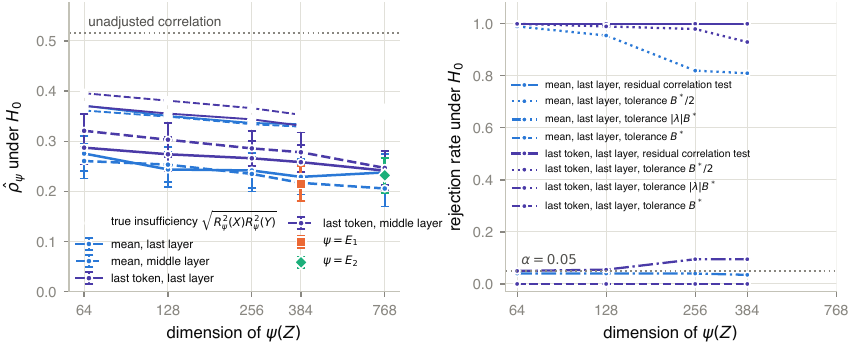}{\linewidth}{5cm}
\caption{Experiment~3, the four embeddings of the generator's state. Left: $\hat\rho_\psi$ under $H_0$ with 95\% intervals against the dimension of $\psi(Z)$, for the principal components of the generator's states (mean over the prompt tokens and last prompt token, at the last and the middle layer), for $E_1$ and $E_2$, and the unadjusted correlation (dotted); hollow markers are the true insufficiency $\sqrt{R^2_\psi(X) R^2_\psi(Y)}$ of each embedding, identified from the replicates. Right: rejection rate under $H_0$ over \resFourTruthR\ null datasets drawn from the replicates, for the residual correlation test and for the robust test with the tolerance equal to half the true insufficiency, to the true bias $|\lambda|\sqrt{R^2_\psi(X) R^2_\psi(Y)}$ and to the true insufficiency ($B^*$ in the legend), for the mean and the last-token state at the last layer. All regressions are least squares.}
\label{fig:exp4views}
\end{figure}

\begin{figure}[h]
\centering
\figorbox{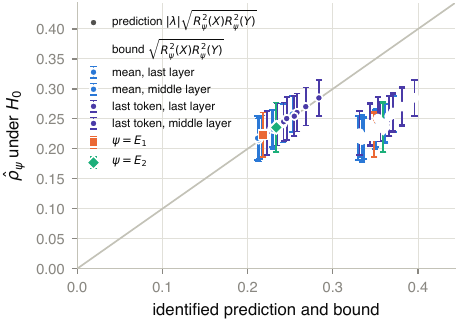}{0.5\linewidth}
\caption{Experiment~3, identification check. For every embedding $\psi$, $\hat\rho_\psi$ under $H_0$ (mean over the \resFourTruthR\ null datasets, with twice its standard deviation) against the prediction $|\lambda|\sqrt{R^2_\psi(X) R^2_\psi(Y)}$ of \cref{thm:pair} (filled) and the bound $\sqrt{R^2_\psi(X) R^2_\psi(Y)}$ (hollow), both identified from the replicates; the line is the diagonal. Every filled marker lies on the diagonal, to within \resFourTruthPredDiffMax, and every hollow marker to its right.}
\label{fig:exp4identity}
\end{figure}

\begin{figure}[h]
\centering
\figorboxh{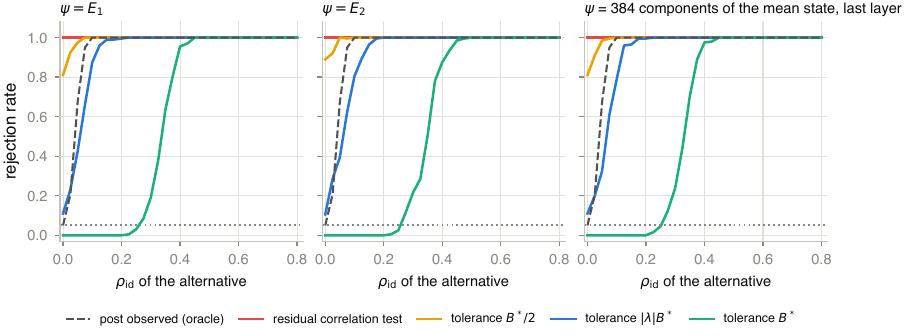}{\linewidth}{3.6cm}
\caption{Experiment~3, rejection rates against the coupled alternatives for $\psi = E_1$, $E_2$ and 384 principal components of the generator's mean state at the last layer, by the conditional correlation $\rho_{\id}$ of comment and summary given the post, over \resFourPowerR\ datasets per value: the residual correlation test, the robust test with the tolerance equal to half the true insufficiency, to the true bias and to the true insufficiency ($B^*$ in the legend), and the reference test that conditions on the post through the unused draws (\cref{app:impl-exp3}).}
\label{fig:exp4power}
\end{figure}

\textbf{Application.} \Cref{fig:appviews} shows politeness and log length against score under every embedding of the comment being replied to. In both pairs, every embedded $\hat\rho_\psi$ lies between the unconditional correlation and the baseline, and closer to the former; for politeness, the embedded values range from \resDeepPolEmbMin\ to \resDeepPolEmbMax\ against a baseline of \resDeepPolStrat. The insufficiency of $E_1$ that any other embedding demonstrates is at most \resDeepBenchMax\ over the six pairs. The baseline estimates $\rho_{\id}$ directly, so the gap to it is the bias of the embedded test. For $E_1$, the identity \eqref{eq:pair} with $\rho_{\id}$ from the baseline and $R^2_{E_1}(X)$, $R^2_{E_1}(Y)$ and $\lambda$ from the variance components predicts $\rho_{E_1} = \resDeepPolPredRho$ for politeness and \resDeepLenPredRho\ for log length, against measured values of \resDeepPolEone\ and \resDeepLenEone.

\begin{figure}[h]
\centering
\figorboxh{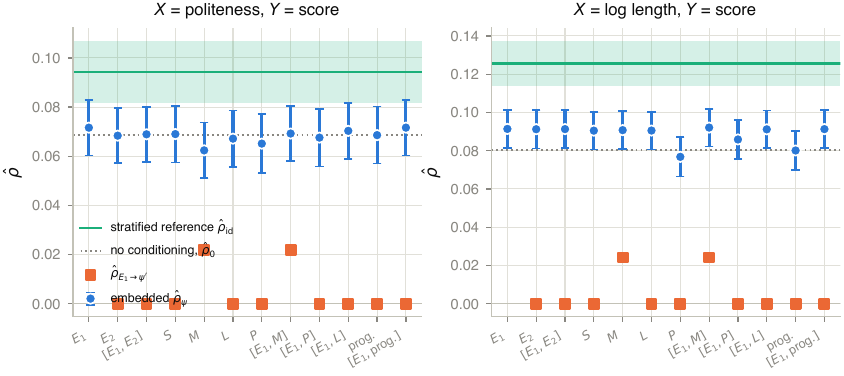}{\linewidth}{5cm}
\caption{Application, politeness against score (left) and log length against score (right) under every embedding of the comment being replied to: $E_1$, $E_2$, $[E_1, E_2]$, the surface statistics $S$, the metadata $M$, the TF-IDF directions $L$, the post embedding $P$, combinations with $E_1$, and the prognostic score. Circles: $\hat\rho_\psi$ with 95\% intervals; squares: the insufficiency of $E_1$ demonstrated by the added embedding, $\hat\rho_{E_1 \to \psi'}$; line and band: the baseline with its 95\% interval; dotted: the unconditional correlation.}
\label{fig:appviews}
\end{figure}

\clearpage
\section{Language model prompts}\label{app:prompts}

Every call in Experiment~3 uses the system prompt below and one of the two instructions, with the post inserted in place of \texttt{\{Z\}}. The comment instruction inserts one of three personas, drawn uniformly at random per post: ``who tends to be critical and skeptical'', ``who tends to be appreciative and enthusiastic'', or ``who is matter-of-fact and reserved''. The two calls for a post share nothing but the model weights, so $X_i \indep Y_i \mid Z_i$ holds conditional on the model (\cref{rem:llm}). The sampling settings are in \cref{app:impl-exp3}.

\textbf{System prompt.} \emph{You write short pieces of text for a research dataset on online discussions. Follow the instructions exactly. Write in natural, varied English. Output plain text only: no preamble, no headings, no quotation marks around the whole answer.}

\textbf{Comment.} \emph{Below is a post from an online discussion forum. POST: \texttt{<{}<{}<} \{Z\} \texttt{>{}>{}>} Write a reader comment of 2-3 sentences reacting to this post, as a forum member \{PERSONA\} would. Then, on a new final line, write exactly one of the words POSITIVE, NEGATIVE or NEUTRAL to label the overall sentiment of your comment.}

\textbf{Summary.} \emph{Below is a post from an online discussion forum. POST: \texttt{<{}<{}<} \{Z\} \texttt{>{}>{}>} Write a 2-3 sentence summary of this post for a newsletter that digests forum discussions. Output only the summary.}

The label line is removed from the comment before embedding. Of the \resFourN\ comments with a label, 1815 report a positive, 1140 a negative and 44 a neutral sentiment.
 

\begin{thebibliography}{41}
\providecommand{\natexlab}[1]{#1}
\providecommand{\url}[1]{\texttt{#1}}
\expandafter\ifx\csname urlstyle\endcsname\relax
  \providecommand{\doi}[1]{doi: #1}\else
  \providecommand{\doi}{doi: \begingroup \urlstyle{rm}\Url}\fi

\bibitem[{Anthropic}(2025)]{anthropic2025haiku}
{Anthropic}.
\newblock Claude {Haiku} 4.5 system card.
\newblock \url{https://www.anthropic.com/system-cards}, 2025.
\newblock Model \texttt{claude-haiku-4-5} through the Anthropic API.

\bibitem[Armstrong and Koles{\'a}r(2018)]{armstrong2018}
Timothy~B. Armstrong and Michal Koles{\'a}r.
\newblock Optimal inference in a class of regression models.
\newblock \emph{Econometrica}, 86\penalty0 (2):\penalty0 655--683, 2018.

\bibitem[Armstrong and Koles{\'a}r(2021)]{armstrong2021}
Timothy~B. Armstrong and Michal Koles{\'a}r.
\newblock Sensitivity analysis using approximate moment condition models.
\newblock \emph{Quantitative Economics}, 12\penalty0 (1):\penalty0 77--108,
  2021.

\bibitem[Bergen et~al.(2026)Bergen, Sejdinovic, and Didelez]{bergen2026gkcm}
Luca Bergen, Dino Sejdinovic, and Vanessa Didelez.
\newblock The generalised kernel covariance measure.
\newblock \emph{arXiv preprint arXiv:2604.03721}, 2026.

\bibitem[Berrett et~al.(2020)Berrett, Wang, Barber, and Samworth]{berrett2020}
Thomas~B. Berrett, Yi~Wang, Rina~Foygel Barber, and Richard~J. Samworth.
\newblock The conditional permutation test for independence while controlling
  for confounders.
\newblock \emph{Journal of the Royal Statistical Society: Series B},
  82\penalty0 (1):\penalty0 175--197, 2020.

\bibitem[Cand{\`e}s et~al.(2018)Cand{\`e}s, Fan, Janson, and Lv]{candes2018}
Emmanuel Cand{\`e}s, Yingying Fan, Lucas Janson, and Jinchi Lv.
\newblock Panning for gold: `model-{X}' knockoffs for high dimensional
  controlled variable selection.
\newblock \emph{Journal of the Royal Statistical Society: Series B},
  80\penalty0 (3):\penalty0 551--577, 2018.

\bibitem[Chang et~al.(2020)Chang, Chiam, Fu, Wang, Zhang, and
  Danescu-Niculescu-Mizil]{chang2020convokit}
Jonathan~P. Chang, Caleb Chiam, Liye Fu, Andrew~Z. Wang, Justine Zhang, and
  Cristian Danescu-Niculescu-Mizil.
\newblock {ConvoKit}: A toolkit for the analysis of conversations.
\newblock In \emph{Proceedings of the 21st Annual Meeting of the Special
  Interest Group on Discourse and Dialogue}, pages 57--60, 2020.

\bibitem[Chernozhukov et~al.(2022)Chernozhukov, Cinelli, Newey, Sharma, and
  Syrgkanis]{chernozhukov2022lss}
Victor Chernozhukov, Carlos Cinelli, Whitney Newey, Amit Sharma, and Vasilis
  Syrgkanis.
\newblock Long story short: omitted variable bias in causal machine learning.
\newblock \emph{arXiv preprint arXiv:2112.13398}, 2022.

\bibitem[Cinelli and Hazlett(2020)]{cinelli2020}
Carlos Cinelli and Chad Hazlett.
\newblock Making sense of sensitivity: extending omitted variable bias.
\newblock \emph{Journal of the Royal Statistical Society: Series B},
  82\penalty0 (1):\penalty0 39--67, 2020.

\bibitem[Egami et~al.(2022)Egami, Fong, Grimmer, Roberts, and
  Stewart]{egami2022}
Naoki Egami, Christian~J. Fong, Justin Grimmer, Margaret~E. Roberts, and
  Brandon~M. Stewart.
\newblock How to make causal inferences using texts.
\newblock \emph{Science Advances}, 8\penalty0 (42):\penalty0 eabg2652, 2022.

\bibitem[Feder et~al.(2022)Feder, Keith, Manzoor, Pryzant, Sridhar,
  Wood-Doughty, Eisenstein, Grimmer, Reichart, Roberts, Stewart, Veitch, and
  Yang]{feder2022}
Amir Feder, Katherine~A. Keith, Emaad Manzoor, Reid Pryzant, Dhanya Sridhar,
  Zach Wood-Doughty, Jacob Eisenstein, Justin Grimmer, Roi Reichart,
  Margaret~E. Roberts, Brandon~M. Stewart, Victor Veitch, and Diyi Yang.
\newblock Causal inference in natural language processing: estimation,
  prediction, interpretation and beyond.
\newblock \emph{Transactions of the Association for Computational Linguistics},
  10:\penalty0 1138--1158, 2022.

\bibitem[Hansen(2008)]{hansen2008}
Ben~B. Hansen.
\newblock The prognostic analogue of the propensity score.
\newblock \emph{Biometrika}, 95\penalty0 (2):\penalty0 481--488, 2008.

\bibitem[Hardt et~al.(2016)Hardt, Price, and Srebro]{hardt2016equality}
Moritz Hardt, Eric Price, and Nati Srebro.
\newblock Equality of opportunity in supervised learning.
\newblock In \emph{Advances in Neural Information Processing Systems},
  volume~29, pages 3315--3323. Curran Associates, Inc., 2016.
\newblock URL
  \url{https://proceedings.neurips.cc/paper_files/paper/2016/file/6a9659feb1216f14f7384ba499518b38-Paper.pdf}.

\bibitem[Imbens(2003)]{imbens2003}
Guido~W. Imbens.
\newblock Sensitivity to exogeneity assumptions in program evaluation.
\newblock \emph{American Economic Review}, 93\penalty0 (2):\penalty0 126--132,
  2003.

\bibitem[Imbens and Manski(2004)]{imbens2004}
Guido~W. Imbens and Charles~F. Manski.
\newblock Confidence intervals for partially identified parameters.
\newblock \emph{Econometrica}, 72\penalty0 (6):\penalty0 1845--1857, 2004.

\bibitem[Keith et~al.(2020)Keith, Jensen, and O'Connor]{keith2020}
Katherine~A. Keith, David Jensen, and Brendan O'Connor.
\newblock Text and causal inference: a review of using text to remove
  confounding from causal estimates.
\newblock In \emph{Proceedings of the 58th Annual Meeting of the Association
  for Computational Linguistics}, pages 5332--5344, 2020.

\bibitem[Kim et~al.(2022)Kim, Neykov, Balakrishnan, and
  Wasserman]{kim2022local}
Ilmun Kim, Matey Neykov, Sivaraman Balakrishnan, and Larry Wasserman.
\newblock Local permutation tests for conditional independence.
\newblock \emph{The Annals of Statistics}, 50\penalty0 (6):\penalty0
  3388--3414, 2022.

\bibitem[Kook and Lundborg(2024)]{kook2024}
Lucas Kook and Anton~Rask Lundborg.
\newblock Algorithm-agnostic significance testing in supervised learning with
  multimodal data.
\newblock \emph{Briefings in Bioinformatics}, 25\penalty0 (6):\penalty0
  bbae475, 2024.

\bibitem[Lang(1995)]{lang1995}
Ken Lang.
\newblock {NewsWeeder}: learning to filter netnews.
\newblock In \emph{Proceedings of the 12th International Conference on Machine
  Learning}, pages 331--339, 1995.

\bibitem[Lundborg et~al.(2024)Lundborg, Kim, Shah, and
  Samworth]{lundborg2024pcm}
Anton~Rask Lundborg, Ilmun Kim, Rajen~D. Shah, and Richard~J. Samworth.
\newblock The projected covariance measure for assumption-lean variable
  significance testing.
\newblock \emph{The Annals of Statistics}, 52\penalty0 (6):\penalty0
  2851--2878, 2024.

\bibitem[Neykov et~al.(2021)Neykov, Balakrishnan, and Wasserman]{neykov2021}
Matey Neykov, Sivaraman Balakrishnan, and Larry Wasserman.
\newblock Minimax optimal conditional independence testing.
\newblock \emph{The Annals of Statistics}, 49\penalty0 (4):\penalty0
  2151--2177, 2021.

\bibitem[Ogburn and VanderWeele(2012)]{ogburn2012}
Elizabeth~L. Ogburn and Tyler~J. VanderWeele.
\newblock On the nondifferential misclassification of a binary confounder.
\newblock \emph{Epidemiology}, 23\penalty0 (3):\penalty0 433--439, 2012.

\bibitem[Ogburn and VanderWeele(2013)]{ogburn2013}
Elizabeth~L. Ogburn and Tyler~J. VanderWeele.
\newblock Bias attenuation results for nondifferentially mismeasured ordinal
  and coarsened confounders.
\newblock \emph{Biometrika}, 100\penalty0 (1):\penalty0 241--248, 2013.

\bibitem[Peters et~al.(2016)Peters, B{\"u}hlmann, and
  Meinshausen]{peters2016causal}
Jonas Peters, Peter B{\"u}hlmann, and Nicolai Meinshausen.
\newblock Causal inference by using invariant prediction: identification and
  confidence intervals.
\newblock \emph{Journal of the Royal Statistical Society Series B: Statistical
  Methodology}, 78\penalty0 (5):\penalty0 947--1012, 2016.
\newblock \doi{10.1111/rssb.12167}.

\bibitem[{Qwen Team}(2025)]{qwen2025}
{Qwen Team}.
\newblock {Qwen3} technical report.
\newblock arXiv preprint arXiv:2505.09388, 2025.

\bibitem[Reimers and Gurevych(2019)]{reimers2019}
Nils Reimers and Iryna Gurevych.
\newblock Sentence-{BERT}: sentence embeddings using {S}iamese {BERT}-networks.
\newblock In \emph{Proceedings of the 2019 Conference on Empirical Methods in
  Natural Language Processing}, pages 3982--3992, 2019.

\bibitem[Ren et~al.(2025)Ren, Jin, Xia, Ke, Huang, Xue, Zhang, Guan, and
  Zhou]{ren2025score}
Yixin Ren, Chenghou Jin, Yewei Xia, Li~Ke, Longtao Huang, Hui Xue, Hao Zhang,
  Jihong Guan, and Shuigeng Zhou.
\newblock Score-based generative modeling for conditional independence testing.
\newblock In \emph{Proceedings of the 31st ACM SIGKDD Conference on Knowledge
  Discovery and Data Mining}, 2025.

\bibitem[Rosenbaum(2002)]{rosenbaum2002}
Paul~R. Rosenbaum.
\newblock \emph{Observational Studies}.
\newblock Springer, New York, 2nd edition, 2002.

\bibitem[Scheidegger et~al.(2022)Scheidegger, H{\"o}rrmann, and
  B{\"u}hlmann]{scheidegger2022}
Cyrill Scheidegger, Julia H{\"o}rrmann, and Peter B{\"u}hlmann.
\newblock The weighted generalised covariance measure.
\newblock \emph{Journal of Machine Learning Research}, 23:\penalty0 1--68,
  2022.

\bibitem[Shah and Peters(2020)]{shah2020}
Rajen~D. Shah and Jonas Peters.
\newblock The hardness of conditional independence testing and the generalised
  covariance measure.
\newblock \emph{The Annals of Statistics}, 48\penalty0 (3):\penalty0
  1514--1538, 2020.

\bibitem[Simnacher et~al.(2026{\natexlab{a}})Simnacher, Keilbar, K{\"o}nig,
  Lippert, and Greven]{simnacher2026ecit}
Marco Simnacher, Georg Keilbar, Benjamin K{\"o}nig, Christoph Lippert, and
  Sonja Greven.
\newblock Embedded conditional independence tests for large language model
  generated text with an application to {G}erman parliament speeches.
\newblock \emph{arXiv preprint arXiv:2609.00946}, 2026{\natexlab{a}}.

\bibitem[Simnacher et~al.(2026{\natexlab{b}})Simnacher, Xu, Park, Lippert, and
  Greven]{simnacher2026dncit}
Marco Simnacher, Xiangnan Xu, Hani Park, Christoph Lippert, and Sonja Greven.
\newblock Deep nonparametric conditional independence tests for images.
\newblock \emph{Journal of Machine Learning Research}, 27\penalty0
  (96):\penalty0 1--62, 2026{\natexlab{b}}.

\bibitem[Spirtes et~al.(2000)Spirtes, Glymour, and
  Scheines]{spirtes2000causation}
Peter Spirtes, Clark Glymour, and Richard Scheines.
\newblock \emph{Causation, Prediction, and Search}.
\newblock Adaptive Computation and Machine Learning. MIT Press, Cambridge, MA,
  2nd edition, 2000.
\newblock ISBN 0-262-19440-6.
\newblock \doi{10.7551/mitpress/1754.001.0001}.

\bibitem[Stoye(2009)]{stoye2009}
J{\"o}rg Stoye.
\newblock More on confidence intervals for partially identified parameters.
\newblock \emph{Econometrica}, 77\penalty0 (4):\penalty0 1299--1315, 2009.

\bibitem[Tan et~al.(2016)Tan, Niculae, Danescu-Niculescu-Mizil, and
  Lee]{tan2016}
Chenhao Tan, Vlad Niculae, Cristian Danescu-Niculescu-Mizil, and Lillian Lee.
\newblock Winning arguments: Interaction dynamics and persuasion strategies in
  good-faith online discussions.
\newblock In \emph{Proceedings of the 25th International Conference on World
  Wide Web}, pages 613--624, 2016.

\bibitem[Tansey et~al.(2022)Tansey, Veitch, Zhang, Rabadan, and
  Blei]{tansey2022}
Wesley Tansey, Victor Veitch, Haoran Zhang, Raul Rabadan, and David~M. Blei.
\newblock The holdout randomization test for feature selection in black box
  models.
\newblock \emph{Journal of Computational and Graphical Statistics}, 31\penalty0
  (1):\penalty0 151--162, 2022.

\bibitem[{TypeSafe AI}(2026)]{typesafe2026jev}
{TypeSafe AI}.
\newblock Jev: a {System One} decision model.
\newblock \url{https://typesafe.ai}, 2026.
\newblock Model \texttt{jev-latest} through the TypeSafe API, accessed 30
  September 2026.

\bibitem[Veitch and Zaveri(2020)]{veitch2020sense}
Victor Veitch and Anisha Zaveri.
\newblock Sense and sensitivity analysis: simple post-hoc analysis of bias due
  to unobserved confounding.
\newblock In \emph{Advances in Neural Information Processing Systems},
  volume~33, 2020.

\bibitem[Veitch et~al.(2020)Veitch, Sridhar, and Blei]{veitch2020}
Victor Veitch, Dhanya Sridhar, and David~M. Blei.
\newblock Adapting text embeddings for causal inference.
\newblock In \emph{Proceedings of the 36th Conference on Uncertainty in
  Artificial Intelligence}, pages 919--928, 2020.

\bibitem[Zhang et~al.(2011)Zhang, Peters, Janzing, and
  Sch{\"o}lkopf]{zhang2011kci}
Kun Zhang, Jonas Peters, Dominik Janzing, and Bernhard Sch{\"o}lkopf.
\newblock Kernel-based conditional independence test and application in causal
  discovery.
\newblock In \emph{Proceedings of the 27th Conference on Uncertainty in
  Artificial Intelligence}, pages 804--813, 2011.

\bibitem[Zhang et~al.(2024)Zhang, Huang, Yang, and Shao]{zhang2024doubly}
Yi~Zhang, Linjun Huang, Yun Yang, and Xiaofeng Shao.
\newblock Doubly robust conditional independence testing with generative neural
  networks.
\newblock \emph{arXiv preprint arXiv:2407.17694}, 2024.

\end{thebibliography}
\end{document}